\documentclass[reqno,a4paper]{amsart}
\usepackage{amsmath, amssymb, tikz, xcolor}
\usepackage{bm, bbm}
\usepackage[mathscr]{eucal}
 \usepackage[top=3cm, bottom=3cm, left=3cm, right=3cm]{geometry}

\usepackage[foot]{amsaddr} 
\allowdisplaybreaks
\usepackage{mathtools}
\usepackage{easyReview}

\usepackage[shortlabels]{enumitem}
\setlist{topsep=2pt, itemsep=2pt}

\usepackage{comment}

\usepackage[colorlinks=true,linkcolor=teal, citecolor=teal]{hyperref}

\usepackage[nameinlink,capitalize]{cleveref}
\crefname{subsection}{Subsection}{Subsections}

\crefname{equation}{}{}
\crefname{theo}{Theorem}{Theorems}
\crefname{coro}{Corollary}{Corollaries}
\crefname{prop}{Proposition}{Propositions}
\crefname{lemm}{Lemma}{Lemmas}
\crefname{exam}{Example}{Examples}
\crefname{assum}{Assumption}{Assumptions}

\newtheorem{theo}{Theorem}[section]

\newtheorem{lemm}[theo]{Lemma}
\newtheorem{prop}[theo]{Proposition}

\theoremstyle{definition}
\newtheorem{defi}[theo]{Definition}

\newtheorem{exam}[theo]{Example}
\newtheorem{rema}[theo]{Remark}

\numberwithin{equation}{section}

\newcommand{\bR}{\mathbb R}
\newcommand{\bN}{\mathbb N}

\newcommand{\bD}{\mathbb D}
\newcommand{\bE}{\mathbb E}
\newcommand{\bG}{\mathbb G}

\newcommand{\bP}{\mathbb P}
\newcommand{\bF}{\mathbb F}
\newcommand{\1}{\mathbbm{1}}

\newcommand{\bI}{\mathbb I}

\newcommand{\bfL}{\mathbf L}
\newcommand{\bfU}{\mathbf U}
\newcommand{\bfV}{\mathbf V}
\newcommand{\bfh}{\mathbf h}

\newcommand{\cB}{\mathcal B}

\newcommand{\cF}{\mathcal F}

\newcommand{\cJ}{\mathcal J}
\newcommand{\cG}{\mathcal G}

\newcommand{\cL}{\mathcal L}

\newcommand{\cP}{\mathcal P}

\newcommand{\cX}{\mathcal X}
\newcommand{\cY}{\mathcal Y}

\newcommand{\scrD}{\mathscr D}

\newcommand{\rmF }{\mathrm{F}}

\newcommand{\frb}{\mathfrak{b}}

\newcommand{\frh}{\mathfrak{h}}

\newcommand{\frr}{\mathfrak{r}}

\newcommand{\wt}{\widetilde}

\newcommand{\tran}{\mathsf{T}}

\newcommand{\ep}{\varepsilon}
\newcommand{\e}{\mathrm{e}}

\newcommand{\Leb}{\bm{\lambda}}

\newcommand{\pd}{\partial}

\newcommand{\od}{\mathrm{d}}

\newcommand{\nq}{\!\!}

\newcommand{\gamI}{\gamma}
\newcommand{\gamII}{\gamma}

\newcommand{\bPn}{\mathbb P}
\newcommand{\bEn}{\mathbb E}
\makeatletter
\newcommand*\bigcdot{\mathpalette\bigcdot@{.35}}
\newcommand*\bigcdot@[2]{\mathbin{\vcenter{\hbox{\scalebox{#2}{$\m@th#1\bullet$}}}}}
\makeatother

\newcommand{\bdot}[2]{#1 \hspace{2pt} \boldsymbol{\bigcdot} \hspace{2pt} #2}

\newcommand{\trm}[1]{\textrm{#1}}

\def\b{\big}
\def\B{\Big}
\def\bb{\bigg}
\def\BB{\Bigg}

\def\<{\left<}
\def\>{\right>}

\begin{document}

\title{A random measure approach to reinforcement learning in continuous time}

\author[Christian Bender]{Christian Bender$^{1}$}
\address{$^1$Department of Mathematics, Saarland University, Germany}
\email{bender@math.uni-saarland.de}

\author[Nguyen Tran Thuan]{Nguyen Tran Thuan$^{1,2}$}
\email{nguyen@math.uni-saarland.de}
\address{$^2$Department of Mathematics, Vinh University, 182 Le Duan, Vinh, Nghe An, Viet Nam}
\email{thuannt@vinhuni.edu.vn}

\thanks{}

\date{\today}

\begin{abstract}
	We present a random measure approach for modeling exploration, i.e., the execution of measure-valued controls,  in continuous-time reinforcement learning (RL) with controlled diffusion and jumps. First, we consider the case when sampling the randomized control in continuous time takes place on a discrete-time grid and reformulate the resulting stochastic differential equation (SDE) as an equation driven by suitable random measures. The construction of these random measures makes use of the Brownian motion and the Poisson random measure (which are the sources of noise in the original model dynamics) as well as the additional random variables, which are sampled on the grid for the control execution. Then, we prove a limit theorem for these random measures as the mesh-size of the sampling grid goes to zero, which leads to the \emph{grid-sampling limit SDE} that is jointly driven  by white noise random measures and  a Poisson random measure. We also argue that the grid-sampling limit SDE can substitute the exploratory SDE and the sample SDE of the recent continuous-time RL literature, i.e., it can be applied for the theoretical analysis of exploratory control problems and for the derivation of learning algorithms.

	\bigskip
	
	\noindent \textbf{Keyworks.} Exploratory control; Orthogonal martingale measures; Poisson random measures; Reinforcement learning; Weak convergence. 
	
	\smallskip
	
\end{abstract}

\maketitle


\section{Introduction}

Recent years have seen tremendous progress in the development of reinforcement learning (RL) for systems in continuous time and space, which are formulated in the language of stochastic differential equations (SDEs). The articles \cite{WZZ20, WZ20} constitute an important starting point for the modeling of exploration of the state space in such a framework. Roughly speaking, the exploration mechanism consists of first choosing a relaxed control (which is a policy with values in the set of probability distributions)  and then executing the policy by drawing a sample from the chosen distribution.
Based on a heuristic argument using law of large numbers, Wang et al. \cite{WZZ20} identify the drift and diffusion coefficient, when averaging over many independent executions of the relaxed control, leading to the \emph{exploratory SDE} in a diffusion setting. Regularizing the cost function by adding a running reward for exploration (e.g., in terms of Shannon entropy as in \cite{WZZ20,WZ20}), they come up with a formulation of \emph{exploratory control problems}.

The exploratory control approach of \cite{WZZ20} has been generalized in many directions, including a mean-field setting \cite{FJ22, GXZ22}, regime-switching models \cite{WL24}, and models with jumps \cite{BN23, GLZ24}. A significant part of the literature focuses on exploratory versions of linear-quadratic problems (which are no longer linear-quadratic due to the presence of the regularization term) and on applications to mean-variance portfolio selection, see, e.g., \cite{BN23,DDJ23, GXZ22, WZZ20, WZ20,  WL24}. Moreover, alternatives to the Shannon entropy regularization term have been suggested, see \cite{DJ24, GXZ22, HWZ23, RZ21}. More information about the recent progress in continuous-time RL can be found in the survey article by Zhou \cite{Zh23}.

While the exploratory SDE is tailor-made to adapt   the classical dynamic programming approach and to tackle exploratory control by means of a suitable variant of the Hamilton--Jacobi--Bellman (HJB) equation (see \cite{TZZ22} for a detailed study of the exploratory HJB equation), it cannot be interpreted as the response of the system to a randomized control (i.e, a sample drawn from a given relaxed control). This is due to the averaging effect in its derivation. Hence, trajectories of the exploratory SDE cannot be regarded as observable and, thus, learning algorithms cannot be formulated in terms of (time-discretized) trajectories of the exploratory control, see also the discussion in \cite[p.9]{JZ22}. 

As a way out, Jia and Zhou \cite{JZ22, JZ23} introduce a \emph{sample SDE}, which models the dynamics of the system along a randomized control in continuous time. Based on the sample SDE and martingale criteria for optimality in continuous time, they provide continuous-time versions of several learning algorithms (including temporal-difference learning and $Q$-learning), see also \cite{SB18} for an overview on learning algorithms in the classical framework of Markov decision processes. The continuous-time algorithms of \cite{JZ22, JZ23} are only discretized at the implementation stage, so they follow the ``first-optimize-then-discretize'' methodology. 
However, no explicit construction of the randomization mechanism (for sampling from a given relaxed control) is provided in \cite{JZ22, JZ23}. The latter reference mentions an uncountable family of independent uniform random variables  $(Z_t)_{t\in [0,T]}$ on the unit cube. So the construction of sample SDEs might need to deal with some measurability issues, for which we refer, e.g., to \cite[Proposition 2.1 and Corollary 4.3]{Su06}.
To circumvent such  measurability problems, we adapt some ideas of \cite{STZ23}. We sample the  independent uniform random variables on a finite time-grid only and extend the randomization scheme piecewise constantly to a left-continuous process (which, consequently, becomes predictable). This approach leads to a well-defined SDE (which we call \emph{grid-sampling SDE}), which has a sound interpretation as response of the system to the grid-randomization of a relaxed control. Technically, this is an SDE with random coefficients.

We are mainly interested in the limit dynamics of this grid-sampling SDE, as the mesh-size of the grid tends to zero. To this end, we reformulate it as an SDE with deterministic coefficients driven by random measures which depend on the grid-sampling randomization process. In this way, the additional randomness for policy execution is moved from the integrand to the integrator. Our main result (\cref{thm:limit} below) implies vague convergence of these grid-dependent random measures, as the grid-size converges to zero. Replacing the grid-dependent random measures by their limit measures, we arrive at the \emph{grid-sampling limit SDE}, which we consider as a  natural SDE formulation for RL with state space exploration in continuous time.

Note that we work in a framework with controlled diffusion and controlled jumps in which the SDE under a classical control is driven by a multivariate Brownian motion and a Poisson random measure. In the ``control randomization limit'', i.e. in our formulation of the grid-sampling limit SDE, the Brownian motion is replaced by a family of independent white noise martingale measures (in the sense of \cite{Wa86,KM90}) and the limit Poisson random measures is defined on an extended measurable space to account for the randomization.

 Our weak convergence approach extends the derivation of the exploratory dynamics for mean-variance portfolio selection with jumps in \cite{BN23}. Due to the linear dependence of the diffusion coefficient on the control, the white noise martingale measures  do not show up there but are replaced by a high-dimensional Brownian motion (which features additional components to model the control randomization) in the context of \cite{BN23}, see also \cref{exmp:linear_control}. However, the limit Poisson random measure is essentially the same as in \cite{BN23} in our more general situation.

We also mention that recently the framework of Zhou and coauthors \cite{JZ22,JZ23,WZZ20}  has been extended to the jump-diffusion case by Gao et al. \cite{GLZ24}. They derive in \cite{GLZ24} the infinitesimal generator of the averaged (over independent policy executions) dynamics heuristically by extending the law of large numbers argument from \cite{WZZ20} in order to define an exploratory SDE with jumps. While the jump part features the same structure as in our grid-sampling limit SDE and as in \cite{BN23}, the diffusion part of their exploratory SDE with jumps is driven by a Brownian motion (which can be lower-dimensional than the Brownian motion that drives the original SDE without control randomization). We also mention that the final form of the grid-sampling limit SDE resembles the classical formulation of relaxed control, see, e.g., \cite{Me92} for the case of diffusion control or Chapter 13 in \cite{KD01}. We emphasize, however, that relaxed controls have been introduced as a technical tool for compactification of the control space in the framework of classical control, while the importance of the grid-sampling limit SDE is in its interpretation as limit to the response of the system to  randomized controls.

\subsection*{Structure of this article} In \cref{sec:discussion} we motivate and discuss the main result. After explaining the general setting and discussing several sampling schemes for randomization, we introduce the grid-sampling SDE at the end of  \cref{sec:setting}. In \cref{sec:discrete_sample_random_measure}, we construct some random measures related to grid sampling and reformulate the grid-sampling SDE  as an SDE driven by these random measures. The main limit theorem is stated in \cref{sec:main_thm}, leading to the definition of the grid-sampling limit SDE.

 In \cref{sec:exploratory}, we compare the exploratory SDE of \cite{WZZ20} and the  grid-sampling limit SDE in a simplified setting. It turns out that the solutions to both SDEs share the same probability law, although one is derived by averaging out the policy randomization a-priori, while the other one is obtained in a limit, when one adds more and more randomization noise. A main difference is that our limit theorem combined with stability results for SDEs driven by martingale measures (e.g., Chapter 13 in \cite{KD01}) suggests a joint convergence of SDE and integrator for the grid-sampling limit SDE, while such a result cannot hold for the exploratory SDE. This difference plays a key role in \cref{sec:learning}, where we re-derive the temporal difference TD(0)-algorithm of \cite{JZ22a,JZ22} for policy evaluation in continuous time based on the grid-sampling limit SDE. In doing so, we avoid reference to any kind of idealized sampling that requires independent, identically distributed families of random variables indexed by continuous time for control randomization.

   The proof of the main  theorem (\cref{thm:limit}) will be given in \cref{sec:limit_proof} and relies on a limit theorem for triangular arrays by Jacod and Shiryaev \cite{JS03}. The key step of the proof is contained in \cref{prop:limit-test-function-g}, which implies convergence of the (modified) semimartingale characteristics of the grid-sampling random measures (integrated against a sufficiently large class of test integrands) to the  semimartingale characteristics of the limit random measures.    

Proofs of some technical results and background information on martingale measures are compiled in the appendices.

\subsection*{Notations}
Let $\bN := \{1, 2, \ldots\}$ and $\bR^m_0 : = \bR^m \backslash\{0\}$.  For $a, b \in \bR$, denote $a \vee b : = \max\{a, b\}$ and $a \wedge b : = \min\{a, b\}$ as usual.  We also let $\int_a^b : = \int_{(a, b]}$ and $\int_{\emptyset} = \sum_{i \in \emptyset} := 0$ by convention. Notation $\log$ stands for the natural logarithm.

\subsubsection*{Matrices and functions} In this article, all vectors are interpreted as column matrices. 
For a vector $x$ we use $x^{(i)}$ to denote its $i$-th component. For a matrix $A$, the entry in the $i$-th row and $j$-th column is $A^{(i, j)}$.
Notation $A^\tran$ stands for the transpose of $A$. The collection of real matrices of size $m \times p$ is denoted by $\bR^{m \times p}$ which is equipped with the Euclidean/Frobenius norm $\|A\|_{\rmF} : = \sqrt{\trm{trace}[A^\tran A]}$. For $m \in \bN$, we denote by $I_m$ the identity matrix of the size $m\times m$.

Let $|\cdot|$ denote the Euclidean norm in $\bR^m$.  The open ball in $\bR^m$ centered at $0$ with radius $r>0$ is $B_m(r) : = \{x \in \bR^m : |x| < r\}$. In $\bR^m$ we always employ the Borel $\sigma$-field $\cB(\bR^m)$ induced by the Euclidean norm.

Let $U \in \cB(\bR^d)$. We denote by $B_b(U; \bR^m)$ the family of all Borel measurable functions $f \colon U \to \bR^m$ satisfying $\|f\|_{B_b(U; \bR^m)}: = \sup_{u \in U} |f(u)| <\infty$. For $m =1$, we simply write $B_b(U): = B_b(U; \bR)$.

Notations $\pd_k f$, $\pd^2_{k, l} f$ stand for usual partial derivatives of $f$ with respect to scalar components. Let $\nabla f$ and $\nabla^2 f$ denote the gradient and the Hessian of $f$ respectively. The family $C^2_b(\bR^m)$ consists of all twice continuously differentiable and bounded functions $f \colon \bR^m \to \bR$ with bounded gradient and Hessian.  $C^2_c(\bR^m)$ contains all $f \in C^2_b(\bR^m)$ with compact support. We let $f \in C^{1, 2}([0, T] \times \bR^m)$ if $f$ is (resp. twice) continuously differentiable with respect to $t \in [0, T]$ (resp. to $y \in \bR^m$) and its partial derivatives are jointly continuous.

\subsubsection*{Stochastic basis}
Let $T \in (0, \infty)$. We assume that $(\Omega, \cF, \bF, \bP)$ satisfies the usual conditions, which means that $(\Omega, \cF, \bP)$ is a compete probability space, the filtration $\bF = (\cF_t)_{t \in [0, T]}$ is right-continuous and $\cF_0$ contains all $\bP$-null sets. This allows us to assume that every $\bF$-adapted local martingale has \textit{c\`adl\`ag} (right-continuous with finite left limits) paths. For a random variable $\xi$, the expectation and conditional expectation given a sub-$\sigma$-algebra $\cG \subseteq \cF$, if it exists under $\bP$, is respectively denoted by $\bE[\xi]$ and $\bE[\xi|\cG]$. We also use the notation $\bfL^p(\bP): = \bfL^p(\Omega, \cF, \bP)$.

We write $\cP_{\bF}$ for the predictable $\sigma$-field on $\Omega\times[0,T]$ with respect to the filtration $\bF$ and say that an $\bR^d$-valued stochastic process $X = (X_t)_{t\in [0,T]}$ is $\bF$-predictable, if the map $X \colon \Omega\times[0,T] \rightarrow \bR^d$ is $\cP_\bF/\cB(\bR^d)$-measurable.

For a  c\`adl\`ag process $X = (X_t)_{t \in [0, T]}$, set $\Delta X_t : = X_t - X_{t-}$ for $t \in [0, T]$, where $X_{0-} : = X_0$ and $X_{t-} : = \lim_{t>s \uparrow t} X_s$ for $t \in (0, T]$. For processes $X = (X_t)_{t \in [0, T]}$, $Y = (Y_t)_{t \in [0, T]}$, we write $X = Y$ to indicate that $X_t = Y_t$ for all $t \in [0, T]$ a.s., and the same meaning applied when the relation ``='' is replaced by some other relations such as ``$\le$'', ``$>$'', etc.

We refer to \cite{JS03} for unexplained notions such as semimartingales, (optional) quadratic covariation $[X, Y]$ and predictable quadratic covariation $\<X, Y\>$ of semimartingales $X$, $Y$.

\section{Motivation and discussion of the main result}\label{sec:discussion} 

\subsection{Controlled SDEs with randomized policies}\label{sec:setting}

We  think of the model dynamics as a system with input coefficients ($a,b,\gamma$ below) that depend on a policy  $h$ in feedback form. The output of the system is influenced by the random noise generated by a multivariate Brownian motion $B$ and an independent Poisson random measure $N$. Thus, for a classical (non-randomized) policy $h$, we end up with the dynamics, for $t \in [0, T]$,
  \begin{align}\label{eq:SDE-classical}
  	\od X^h_t & = b(t, X^h_{t-}, h(t,X^h_{t-} )) \od t + a(t, X^{h}_{t-},  h(t,X^h_{t-} )) \od B_t \notag \\
  	& \quad + \int_{0 < |z| \le \mathfrak{r}} \gamI(t, X^h_{t-},  h(t,X^h_{t-} ), z) \tilde N(\od t, \od z) + \int_{|z| > \mathfrak{r}} \gamII(t, X^h_{t-},  h(t,X^h_{t-} ), z) N(\od t, \od z),
  \end{align}
 with initial condition $X^h_0 = x_0 \in \bR^m$. The coefficients $b\colon [0, T] \times \bR^m \times \bR^d \to \bR^m$, $a \colon [0, T] \times \bR^m \times \bR^d \to \bR^{m \times p}$ and $\gamI\colon [0, T] \times  \bR^m \times \bR^d \times \bR^q_0 \to \bR^m$ and the feedback policy $h \colon [0, T] \times \bR^m \to  \bR^d$  are measurable and assumed to be sufficiently regular to guarantee existence of a unique strong solution. Moreover, $B = (B_t)_{t \in [0, T]}$ is a standard $p$-dimensional Brownian motion, $N(\od t, \od z)$ is a (possibly inhomogeneous) Poisson random measure independent of $B$ with intensity $\nu(\od t, \od z) = \nu_t(\od z) \od t$ where $\nu_t$ is a L\'evy measure on $\bR^q_0$ (i.e., $\nu_t$ is a Borel measure  with $\int_{\bR^q_0} (|z|^2 \wedge 1) \nu_t(\od z) < \infty$) for all $t \in [0, T]$. 
Throughout this article, we assume that
\begin{align}\label{assumption:Levy-measure-r}
 \int_0^T \nq \int_{\bR^q_0} ( |z|^2 \1_{\{0 < |z| \le \mathfrak{r}\}}  + \1_{\{|z| > \mathfrak{r}\}}) \nu_t(\od z)  \od t < \infty
\end{align}
for some fixed $\mathfrak{r} \in [0, \infty]$. We may think of $\frr$ as the threshold to distinguish between small jumps and large jumps -- and, as usual, the small jumps are integrated with respect to the compensated random measure $\tilde N(\od t, \od z)=N(\od t, \od z)- \nu_t(\od z) \od t$. Here, the Brownian motion $B$ and the Poisson random measure $N$ are  defined on a filtered probability space $(\Omega, \cF, \bar{\mathbb{F}}, \bP)$ which satisfies the usual conditions. Note that the filtration $\bar{\mathbb{F}}$ may be larger than the one generated by $(B, N)$. 

\begin{rema}\label{rema:levy_integration}
	\begin{enumerate}
			
		\item One typically takes $\mathfrak{r} = 1$ which corresponds to the canonical truncation function $z\1_{\{0 < |z| \le 1\}}$. However, since the random measures are handled differently between the ``compensated jump part'' and the ``finite activity jump part'', we include here the case $\mathfrak{r} = 0$, which means that the jump part $\int_0^{\cdot}  \int_{\bR^q_0} zN(\od t, \od z)$ of the driving inhomogeneous L\'evy process is of finite activity, and the case $\mathfrak{r} = \infty$ which means that the jump part $\int_0^{\cdot}  \int_{\bR^q_0} z \tilde N(\od t, \od z)$  is a square integrable martingale.
		
		\item Note that \eqref{assumption:Levy-measure-r} holds for some $\mathfrak{r} \in (0, \infty)$ if and only if \eqref{assumption:Levy-measure-r} holds for all $\mathfrak{r} \in (0, \infty)$.
	\end{enumerate}
\end{rema}

A \emph{relaxed (or, measure-valued) control} in feedback form is a mapping ${h}\colon [0,T]\times \bR^m \to \mathcal{P}r(\cB(\bR^d))$, where $\mathcal{P}r(\cB(\bR^d))$ denotes the space of probability measures on the Borel field $\cB(\bR^d)$ over $\bR^d$. For the execution of a relaxed control, we consider an $\bar{\mathbb{F}}$-predictable stochastic process $\xi = (\xi_t)_{t\in [0,T]}$ independent of $(B,N)$, whose marginal distribution $\xi_t$ is a uniform distribution on $[0,1]^d$ for every $t\in [0,T]$.  Such a $\xi=(\xi_t)_{t\in [0,T]}$ will be called a \emph{randomization process}. We think of a measurable function 
${\bfh}\colon[0,T]\times \bR^m \times [0,1]^d \rightarrow \bR^d$ as a \emph{randomized control} in feedback form. The actual randomization is performed by plugging a randomization process in the last variable of ${\bfh}$. Adapting the terminology in \cite{STZ23} to our setting, we say that a randomized control   ${\bfh}$ \emph{executes} a relaxed control $h$, if the random variable  ${\bfh}(t,x,\xi_{t})$ has the distribution $h(t,x)$ for every $t\in [0,T]$ and $x\in \bR^m$ (for some, and then for any, randomization process $\xi$). For a given randomization process $\xi$, the random field $(\bfh(t,x,\xi_t))_{t\in [0,T],\,x\in \bR^m}$ will be called a \textit{$(\xi_t)_{t\in [0,T]}$-randomized policy}.

\begin{rema}
	\begin{enumerate}
		\item We have only fixed the marginal distribution of the randomization process  $(\xi_t)_{t\in [0,T]}$, but not the joint distribution. In particular, $\xi_s$ and $\xi_t$ are, for the moment, not supposed to be independent for $s\neq t$. Several constructions of the process $\xi$ will be discussed below.
		\item It is well known that for every distribution $P$ on $\cB(\bR^d)$, there is a measurable function $H$ such that $H(\eta)$ is $P$-distributed for any uniform random variable $\eta$ on  $[0,1]^d$. This is one motivation to assume that the marginals of $\xi$ are uniformly distributed. Note, however, that for any vector $(\eta_1,\ldots, \eta_d)$ of independent standard Gaussian random variables, the vector $(\Phi(\eta_1),\ldots,\Phi(\eta_d) )$ is uniformly distributed on $[0,1]^d$. Here, $\Phi$ denotes the cumulative distribution function of a standard Gaussian. Hence, changing the marginal distribution of $(\xi_t)_{t\in [0,T]}$, e.g.,  to a multivariate Gaussian as in \cite{BN23} does not make any essential difference in the constructions to come.
	\end{enumerate}
\end{rema}
The crucial property of the randomization process $\xi = (\xi_t)_{t\in [0,T]}$ is its predictability which necessarily implies the predictability of the random field $({\bfh}(t,x,\xi_t))_{t\in [0,T],\,x\in \bR^m}$. Hence, for a randomized control ${\bfh}$ and a fixed randomization process $\xi$, it makes sense to consider the random coefficient SDE
   \begin{align}\label{eq:SDE-random}
	\od  X^{\xi,\bfh}_t & = b(t,  X^{\xi,\bfh}_{t-}, {\bfh}(t, X^{\xi,\bfh}_{t-},\xi_t )) \od t + a(t,  X^{\xi,\bfh}_{t-},  {\bfh}(t, X^{\xi,\bfh}_{t-},\xi_t )) \od B_t \notag \\
	& \quad + \int_{0 < |z| \le \mathfrak{r}} \gamI(t,  X^{\xi,\bfh}_{t-},  {\bfh}(t, X^{\xi,\bfh}_{t-},\xi_t ), z) \tilde N(\od t, \od z) \notag \\ &\quad + \int_{|z| > \mathfrak{r}} \gamII(t,  X^{\xi,\bfh}_{t-},  {\bfh}(t, X^{\xi,\bfh}_{t-},\xi_t ), z) N(\od t, \od z),
\end{align}
which describes the dynamics of the system along the $(\xi_t)_{t\in [0,T]}$-randomized feedback policy $({\bfh}(t,x,\xi_t))_{t\in [0,T],\,x\in \bR^m}$. 

We next discuss two approaches for $\xi=(\xi_t)_{t\in [0,T]}$:
\begin{itemize}
	\item \emph{Idealized sampling:} In idealized sampling, the family $(\xi_t)_{t\in [0,T]}$ of random variables is assumed to be independent. Note that there is no problem to construct the triplet $(B,N,\xi)$ on an appropriate product space. It is, however, known that a family of non-constant independent identically distributed random variables  $(\xi_t)_{t\in [0,T]}$ cannot be realized in a jointly measurable way with respect to the standard product $\sigma$-field, i.e., the map $\xi \colon \Omega\times[0,T]\rightarrow [0,1]^d$ cannot be $\cF\otimes\cB([0,T])/\cB([0,1]^d)$-measurable, see, e.g., Proposition 2.1 in \cite{Su06} and the detailed discussion on the relevance of the results in \cite{Su06} for policy execution in \cite{STZ23}.  In particular, with idealized sampling, we can never obtain the crucial predictability property of   $(\xi_t)_{t\in [0,T]}$, and, hence, it is not clear how to make any good sense of the SDE \eqref{eq:SDE-random} for a sufficiently large class of $(\xi_t)_{t\in[0,T]}$-randomized policies. While no explicit construction of the policy execution for the sample SDE in \cite{JZ23} is provided, the authors introduce an uncountable family of independent uniform random variables for performing the policy execution. Thus, their sample SDE could face the measurability issue detailed above. 
	\item \emph{Grid-sampling:} Let $\Pi$ be a partition of $[0,T]$ with grid points $0=t_0<t_1<\cdots<t_n=T$ for some $n\in \bN$ and mesh-size $|\Pi|:=\max_{1 \le i \le n}|t_i-t_{i-1}|$. Assuming that the probability space carries an independent family $(\xi_1,\ldots, \xi_n)$ of uniforms on $[0,1]^d$ independent of $(B,N)$, we consider 
	the randomization process $\xi^\Pi = (\xi^\Pi_t)_{t \in [0, T]}$ given by
	\begin{align*}
	\xi^\Pi_t :=\sum_{j=1}^n \xi_j {\bf 1}_{(t_{j-1},t_j]}(t),\quad t\in [0,T].
\end{align*}
Writing $\bF^\Pi = (\cF^\Pi_t)_{t \in [0, T]}$ for the right-continuous, augmented version of the filtration generated by $(B,N,\xi^\Pi)$, the process $\xi^\Pi$ is left-continuous and adapted, hence $\bF^\Pi$-predictable. Remark that $\xi^{\Pi}_{t_i}=\xi_i$ is $\cF^\Pi_{t_{i-1}}$-measurable, but independent of $\cF^\Pi_{(t_{i-1})-}$, and $B$ and $N(\od t, \od z)$ are still a Brownian motion and a Poisson random measure with intensity $\nu_t(\od z)\od t$ with respect to $\bF^\Pi$.

We emphasize that the authors in \cite{STZ23} and \cite{GRZ24} have already applied this type of grid-sampling as a substitution for the infeasible idealized sampling when executing Gaussian relaxed policies in the context of linear-quadratic control.
\end{itemize}

By predictability of the grid-sampling process $\xi^\Pi$, we may consider the SDE \eqref{eq:SDE-random} with $\xi=\xi^\Pi$ and we call this SDE the \textit{grid-sampling SDE} along the randomization process $\xi^\Pi$. Remark that it can be solved iteratively on the subintervals of the partition under standard Lipschitz and growth assumptions, i.e., for $t\in (t_{i-1},t_i]$,
 \begin{align}\label{eq:SDE-Pi_sampling}
	X^{\Pi,\bfh}_t & = X^{\Pi,\bfh}_{t_{i-1}}+ \int_{t_{i-1}}^t b(s, X^{\Pi,\bfh}_{s-}, {\bfh}(s,X^{\Pi,\bfh}_{s-},\xi^\Pi_{t_i} )) \od s + \int_{t_{i-1}}^t a(s, X^{\Pi,\bfh}_{s-},  {\bfh}(s,X^{\Pi,\bfh}_{s-},\xi^\Pi_{t_i} )) \od B_s \notag \\
	& \quad + \int_{(t_{i-1},t]} \int_{0 < |z| \le \mathfrak{r}} \gamI(s, X^{\Pi,\bfh}_{s-},  {\bfh}(s,X^{\Pi,\bfh}_{s-},\xi^\Pi_{t_i} ), z) \tilde N(\od s, \od z)\notag \\ & \quad + \int_{(t_{i-1},t]}\int_{|z| > \mathfrak{r}} \gamII(s, X^{\Pi,\bfh}_{s-},  {\bfh}(s,X^{\Pi,\bfh}_{s-},\xi^\Pi_{t_i} ), z) N(\od s, \od z),
\end{align}
see, e.g., Theorem IV.9.1 in \cite{IW89} for the case of a homogeneous Poisson random measure.
\begin{rema}
Suppose that the randomized control ${\bfh}$ is continuous and executes the relaxed control $h$ and that the sampling grid $\Pi$ is ``sufficiently fine''. Then, we may consider 
$$
{\bfh}(t_{i-1},X^{\Pi,\bfh}_{t_{i-1}},\xi^\Pi_{t_i} ) =\lim_{s\searrow t_{i-1}} {\bfh}(s,X^{\Pi,\bfh}_{s-},\xi^\Pi_{t_i} ) 
$$
as a ``good'' approximation to  ${\bfh}(s,X^{\Pi,\bfh}_{s-},\xi^\Pi_{t_i} )$ for $s\in (t_{i-1},t_i]$.	Note that $X^{\Pi,\bfh}_{t_{i-1}}=X^{\Pi,\bfh}_{(t_{i-1})-}$ a.s. Thus, $X^{\Pi,\bfh}_{t_{i-1}}$ is $\cF^\Pi_{(t_{i-1})-}$-measurable and, consequently, independent of $\xi^\Pi_{t_i}$. Therefore, we can interpret the approximation ${\bfh}(t_{i-1},X^{\Pi,\bfh}_{t_{i-1}},\xi^\Pi_{t_i} )$ in the following way: The actor first chooses the distribution $h(t_{i-1},X^{\Pi,\bfh}_{t_{i-1}} )$ and, then, the independent uniform random variable $\xi^\Pi_{t_i}$ is generated to sample from this distribution.
\end{rema}

\subsection{Random measure interpretation of grid-sampling}\label{sec:discrete_sample_random_measure}

We are interested in the limit dynamics of the grid-sampling SDE \eqref{eq:SDE-Pi_sampling} as the mesh-size of the sampling partition $\Pi$ tends to zero. Note that the limit (in finite-dimensional distributions) of grid-sampling scheme leads to idealized sampling. So it does not appear to be promising to pass to the limit on the level of the random coefficients of the grid-sampling SDE. Instead, we change the perspective and consider the SDE \eqref{eq:SDE-Pi_sampling} as a system with deterministic input coefficients $(b,a,\gamI,\bfh)$ which is subjected to the noise given by $(B,N,\xi^\Pi)$. This means that, in this subsection, we first identify suitable random measures depending on $(B,N,\xi^\Pi)$ such that \eqref{eq:SDE-Pi_sampling} can be re-written in the form
\begin{align}\label{eq:SDE-random_measure_Pi}
	\od X^{\Pi,\bfh}_t & = b(t, X^{\Pi,\bfh}_{t-}, {\bfh}(t,X^{\Pi,\bfh}_{t-},u )) M_D^\Pi(\od t,\od u) + a(t, X^{\Pi,\bfh}_{t-},  {\bfh}(t,X^{\Pi,\bfh}_{t-},u )) M_B^\Pi(\od t, \od u) \notag \\
	& \quad + \int_{0 < |z| \le \mathfrak{r}} \gamI(t, X^{\Pi,\bfh}_{t-},  {\bfh}(t,X^{\Pi,\bfh}_{t-},u ), z) \tilde M_J^\Pi(\od t, \od z, \od u) \notag \\ &\quad+ \int_{|z| > \mathfrak{r}} \gamII(t, X^{\Pi,\bfh}_{t-},  {\bfh}(t,X^{\Pi,\bfh}_{t-},u ), z)  M_J^\Pi(\od t, \od z, \od u).
\end{align}
In the next subsection, we will then state our main result on the joint convergence of the random measures $(M_D^\Pi, M_B^\Pi, M_J^\Pi)$ as the mesh-size of $\Pi$ tends to zero. The limit random measures can finally be used to define a meaningful limit SDE of the grid-sampling SDE.

\medskip

\noindent\textit{\textbf{For the drift part}}: We consider 
\begin{equation}\label{eq:drift_Pi_measure}
	M_D^\Pi(\omega,\od t,\od u) := \sum_{i=1}^n  \1_{(t_{i-1}, t_i]}(t) \delta_{\xi^\Pi_{t_i}(\omega)}(\od u)\od t,
\end{equation}
where $\delta_y$ denotes the Dirac distribution on the point $y$. Then, $M^\Pi_D$  is a random measure in the sense of \cite[Definition II.1.3]{JS03}. The following lemma, which links integration with respect to $M_D^\Pi$ to the drift part of the grid-sampling SDE, is straightforward to prove.
\begin{lemm}\label{lem:drift_Pi}
A	
measurable random field $Y \colon \Omega\times [0,T]\times [0,1]^d\rightarrow \bR$ is integrable with respect to $M_D^\Pi$, if and only if
$$
\sum_{i=1}^n	\int_{t_{i-1}}^{t_i} |Y_s(\xi^\Pi_{t_i})| \od s<\infty,\quad \bP\textnormal{-a.s.}
$$
In this case, a.s.,
$$
\int_{(0,T]\times[0,1]^d} Y_s(u) M^\Pi_D(\od s,\od u)= \sum_{i=1}^n\int_{t_{i-1}}^{t_i} Y_s(\xi^\Pi_{t_i}) \od s.
$$
\end{lemm}

\medskip

\noindent\textit{\textbf{For the Brownian part}}: We define
\begin{equation}\label{eq:Bm_Pi_measure}
	M_{B^{(l)}}^\Pi(\omega,t,A) := \bb(\int_0^{t}  \sum_{i=1}^n \1_{(t_{i-1},{t_i}]}(s) \1_A(\xi^\Pi_{t_i})\; \od B^{(l)}_s \bb)(\omega),\quad A\in  \cB([0,1]^d),\;t\in [0,T],\; l=1,\ldots, p.
\end{equation}
Note that the integrand is a bounded $\bF^\Pi$-predictable process, and, hence, the It\^o integrals are well defined.
\begin{lemm}\label{lem:Bm_Pi}
	For any $l=1,\ldots, p$, $M_{B^{(l)}}^\Pi$ is an  orthogonal martingale measure on $[0, T] \times \cB([0,1]^d)$ in the sense of \cite{KM90} with intensity measure $M^\Pi_D$. Moreover, for every $\mathbb{F}^\Pi$-predictable (i.e.,  $\mathcal{P}_{{\mathbb{F}^\Pi}}\otimes \cB([0,1]^d)/\cB(\bR)$-measurable) random field $Y \colon \Omega\times [0,T]\times [0,1]^d\rightarrow \bR$ satisfying 
	$$
\sum_{i=1}^n	\bE\bb[\int_{t_{i-1}}^{t_i} |Y_s(\xi^\Pi_{t_i})|^2 \od s \bb]<\infty,
	$$
	$Y$ can be integrated against $M_{B^{(l)}}^\Pi$ and, a.s.,
	\begin{equation}  \label{eq:lem:Bm_Pi}
	\int_{(0,T]\times[0,1]^d} Y_s(u) M_{B^{(l)}}^\Pi(\od s,\od u) = \sum_{i=1}^n \int_{t_{i-1}}^{t_i}  Y_s(\xi^\Pi_{t_i}) \od B^{(l)}_s.
		\end{equation} 
\end{lemm}
Background information on orthogonal martingale measures, including a review of the integration theory, and a proof of \cref{lem:Bm_Pi} can be found in \cref{app:random_measures}.

\medskip

\noindent \textit{\textbf{For the jump part}}: We first consider the (inhomogeneous) purely non-Gaussian L\'evy process 
\begin{align}\label{def:induced-Levy-process}
	L_t=\int_{(0,t]}\int_{0<|z|\le \mathfrak{r}} z \tilde N(\od s,\od z) + \int_{(0,t]}\int_{|z|> \mathfrak{r}} z  N(\od s,\od z), \quad t\in [0,T],
\end{align}
and recall that $N$ is the jump measure of $L$, i.e.,
$$
N(\omega,\od t,\od z)=\sum_{t\in (0,T]} \1_{\{\Delta L_t(\omega)\neq  0\}} \delta_{(t,\Delta L_t(\omega))}(\od t,\od z).
$$
We now introduce the new integer-valued random measure
\begin{equation} \label{eq:jump_Pi_measure}
M_J^\Pi(\omega,\od t,\od z,  \od u) := \sum_{i=1}^n  \sum_{t\in (t_{i-1},t_i]} \1_{\{\Delta L_t(\omega)\neq  0\}} \delta_{(t,\Delta L_t(\omega),\xi^\Pi_{t_i}(\omega))}(\od t,\od z, \od u)
\end{equation}
on $[0,T]\times \bR^q_0\times[0,1]^d$. It has the same jump times as $N$, but features an ``extra jump size'' $\xi^\Pi_{t_i}$ in the new variable $u$ for the control randomization, if the jump takes place in the interval $(t_{i-1},t_i]$. As stated in the following lemma, its predictable compensator measure is given by
$$
\mu_J^\Pi(\omega,\od t,\od z,  \od u) := \sum_{i=1}^n  \1_{(t_{i-1}, t_i]}(t) \delta_{\xi^\Pi_{t_i}(\omega)}(\od u) \nu_t(\od z) \od t.
$$
Hence, stochastic integration with respect to the compensated random measure $\tilde M^\Pi_J=M^\Pi_J-\mu_J^\Pi$ can be defined in the sense of  \cite[Ch.II, §1d]{JS03}.
\begin{lemm}\label{lem:jumps_Pi}
	\begin{enumerate}[\quad \rm (1)]
		\item The random measure $\mu_J^\Pi$ is the $(\bF^{\Pi}, \bP)$-predictable compensator measure of the integer-valued random measure  $M_J^\Pi$.
		\item Suppose that $Y \colon \Omega\times [0,T]\times \bR^q_0\times [0,1]^d\rightarrow \bR$ is an $\bF^{\Pi}$-predictable random field (i.e. $Y$ is $\mathcal{P}_{{\mathbb{F}^\Pi}}\otimes\cB({\bR^q_0})\otimes \cB([0,1]^d)/\cB(\bR)$-measurable). If 
		$$
		\sum_{i=1}^n  \int_{(t_{i-1},{t_i}]\times {\bR^q_0} }  |Y_s(z,\xi^\Pi_{t_i})| N(\od s,\od z)<\infty,\quad \bP\textnormal{-a.s.},
		$$
		 then $Y$ is integrable with respect to $M_J^\Pi$ and, a.s.,
			\begin{equation}  \label{eq:lem:jumps_Pi-1}
		\int_{(0,T]\times \bR^q_0 \times[0,1]^d} Y_s(z,u) M_{J}^\Pi(\od s,\od z,\od u)= \sum_{i=1}^n \int_{(t_{i-1},{t_i}]\times {\bR^q_0}} Y_s(z,\xi^\Pi_{t_i}) N(\od s,\od z).
		\end{equation}
		Moreover, if 
		$$
		\sum_{i=1}^n \bE\bb[	\int_{t_{i-1}}^{t_i} \int_{\bR^q_0} |Y_s(z,\xi^\Pi_{t_i})|^2 \nu_s(\od z)\od s\bb]<\infty,
		$$
		then $Y$ is integrable with respect to $\tilde M_J^\Pi$ and, a.s.,
	\begin{equation}  \label{eq:lem:jumps_Pi-2}
		\int_{(0,T]\times \bR^q_0 \times[0,1]^d} Y_s(z,u) \tilde M_{J}^\Pi(\od s,\od z,\od u)= \sum_{i=1}^n \int_{(t_{i-1},{t_i}]\times {\bR^q_0}}  Y_s(z,\xi^\Pi_{t_i}) \tilde N(\od s,\od z).
		\end{equation}
	\end{enumerate}
	\end{lemm}
Again, the proof can be found in  \cref{sec:proofs-random-measure}.

\smallskip

In view of \cref{lem:drift_Pi,lem:Bm_Pi,lem:jumps_Pi}, we can, indeed, re-write the grid-sampling SDE \eqref{eq:SDE-Pi_sampling} in the form \eqref{eq:SDE-random_measure_Pi}, utilizing the random measures introduced in \eqref{eq:drift_Pi_measure}, \eqref{eq:Bm_Pi_measure}, and \eqref{eq:jump_Pi_measure}. For instance, assuming that $a$ is bounded and applying standard conventions on integration of matrix-valued integrands, the Brownian part in \eqref{eq:SDE-Pi_sampling} becomes for $t\in (t_{i-1},t_i]$,
\begin{align*}
& \int_{t_{i-1}}^t a(s, X^{\Pi,\bfh}_{s-},  {\bfh}(s,X^{\Pi,\bfh}_{s-},\xi^\Pi_{t_i} )) \od B_s \\
& =\left(\begin{matrix} \sum_{l=1}^p \int_{t_{i-1}}^t a^{(1,l)}(s, X^{\Pi,\bfh}_{s-},  {\bfh}(s,X^{\Pi,\bfh}_{s-},\xi^\Pi_{t_i} )) \od B^{(l)}_s \\ \vdots \\ \sum_{l=1}^p \int_{t_{i-1}}^t a^{(m,l)}(s, X^{\Pi,\bfh}_{s-},  {\bfh}(s,X^{\Pi,\bfh}_{s-},\xi^\Pi_{t_i} )) \od B^{(l)}_s \end{matrix}	\right) \\[4pt]
& =
\left(\begin{matrix} \sum_{l=1}^p 	\int_{(t_{i-1},t]\times[0,1]^d}  a^{(1,l)}(s, X^{\Pi,\bfh}_{s-},  {\bfh}(s,X^{\Pi,\bfh}_{s-},u )) M_{B^{(l)}}^\Pi(\od s,\od u)\\ \vdots \\ \sum_{l=1}^p 	\int_{(t_{i-1},t]\times[0,1]^d}  a^{(m,l)}(s, X^{\Pi,\bfh}_{s-},  {\bfh}(s,X^{\Pi,\bfh}_{s-},u )) M_{B^{(l)}}^\Pi(\od s,\od u) \end{matrix}	\right) \\[4pt]
& =  \int_{(t_{i-1},t]\times[0,1]^d}   a(s, X^{\Pi,\bfh}_{s-},  {\bfh}(s,X^{\Pi,\bfh}_{s-},u )) M_{B}^\Pi(\od s,\od u),
\end{align*}
where $M^\Pi_B = (M^\Pi_{B^{(1)}}, \ldots, M^\Pi_{B^{(p)}})^\tran$.

\subsection{Limit theorem and grid-sampling limit SDE}
\label{sec:main_thm}
In this subsection, we establish a limit theorem for the random measures $(M_D^\Pi,M_B^\Pi, M_J^\Pi)$ defined in \eqref{eq:drift_Pi_measure}--\eqref{eq:jump_Pi_measure}, which drive the grid-sampling SDE \eqref{eq:SDE-random_measure_Pi}, as the mesh-size of the partition $\Pi$ goes to zero. This limit theorem suggests a formulation for the \emph{grid-sampling limit SDE}, which replaces $(M_D^\Pi,M_B^\Pi, M_J^\Pi)$  by the limit random measures $(M_D,M_B, M_J)$ in \eqref{eq:SDE-random_measure_Pi}.

\smallskip

We define
$$
M_D(A) := \Leb_{[0,T]}\otimes \Leb_{[0,1]}^{\otimes d} (A),\quad A\in \cB([0,T])\otimes \cB([0,1]^d),
$$
where $\Leb_U$ stands for the restriction of the Lebesgue measure to a Borel set $U$. Moreover, we let $(M_{B^{(1)}},\ldots, M_{B^{(p)}})$ denote $p$ independent martingale measures with continuous paths and intensity measure $M_D$. Continuous martingale measures with deterministic intensities are also called white noise martingale measures, and we refer to \cite{KM90} for a construction of such martingale measures and more background information. \cref{lem:Bm_from_white_noise} below provides some information on their relation to Brownian motion. 

Finally, $M_J$ denotes a Poisson random measure  on $[0,T] \times \bR^q_0\times [0,1]^d$ with intensity measure
$$
\mu_J(\od t, \od z,\od u) := \nu_t(\od z)\od u \od t.
$$
 An explicit construction of $M_J$ can be found in  \cref{sec:random_measure_construction}. As usual $\tilde M_J=M_J-\mu_J$ stands for the compensated   Poisson random measure.  
 
 We assume  that the original filtered probability space $(\Omega, \cF, \bar{\mathbb{F}}, \bP)$ has been chosen sufficiently large to carry  $(M_{B^{(1)}},\ldots, M_{B^{(p)}})$ and  $M_J$. By \cref{rema:independence}, $(M_{B^{(1)}},\ldots, M_{B^{(p)}})$ and  $M_J$ are automatically independent.  We denote by $\bF$ the  right-continuous, augmented version of the filtration generated by $(M_{B^{(1)}},\ldots, M_{B^{(p)}}, M_J)$.
\begin{theo}\label{thm:limit}
	Let $(\Pi_n)_{n\in \bN}$ be a sequence of finite partitions of $[0,T]$ with $\lim_{n\rightarrow \infty} |\Pi_n|=0$.  For any $m\in \bN$, $R\in(0,\infty)\cup\{\mathfrak{r}\}$, and for any bounded measurable functions $f_{l}^{(k)}\colon [0,T] \times [0,1]^d\to \bR$ ($l=0,\ldots,p$; $k=1,\ldots,m$),   $f_{l}^{(k)} \colon [0,T]  \times \bR^q_0 \times [0,1]^d\to \bR$ ($l=p+1,p+2$; $k=1,\ldots,m$),
	consider the sequence of $\bR^m$-valued processes $\cX^n=(\cX^{n, (1)},\ldots, \cX^{n,(m)})$  defined via
	\begin{align*}
		\cX^{n,(k)}_t &= \int_{(0,t]\times[0,1]^d} f_{0}^{(k)}(s,u) M_D^{\Pi_n}(\od s,\od u)+ \sum_{l=1}^p \int_{(0,t]\times[0,1]^d} f_{l}^{(k)}(s,u) M_{B^{(l)}}^{\Pi_n}(\od s,\od u)  \\ 
		& \quad +  	\int_{(0,t]\times \{ 0<|z|\leq R\} \times[0,1]^d} f_{p+1}^{(k)}(s,z,u)|z| \tilde M_{J}^{\Pi_n}(\od s,\od z,\od u) \\ 
		& \quad +  	\int_{(0,t]\times \{ |z|> R\} \times[0,1]^d} f_{p+2}^{(k)}(s,z,u) M_{J}^{\Pi_n}(\od s,\od z,\od u), \quad t\in [0,T] ,\; k=1,\ldots, m.
	\end{align*}
	Then, $(\cX^n)_{n\in \bN}$ converges weakly in the Skorokhod topology on the space $\bD_T(\bR^m)$ of $\bR^m$-valued, c\`adl\`ag functions to $\cX=(\cX^{(1)},\ldots, \cX^{(m)} )$, where
	\begin{align*}
		\cX^{(k)}_t &= \int_{(0,t]\times[0,1]^d} f_{0}^{(k)}(s,u) M_D(\od s,\od u)+ \sum_{l=1}^p \int_{(0,t]\times[0,1]^d}f_{l}^{(k)}(s,u) M_{B^{(l)}}(\od s,\od u)  \\ 
		& \quad +  	\int_{(0,t]\times \{ 0<|z|\leq R\} \times[0,1]^d} f_{p+1}^{(k)}(s,z,u)|z| \tilde M_{J}(\od s,\od z,\od u) \\ 
		& \quad +  	\int_{(0,t]\times \{ |z|> R\} \times[0,1]^d} f_{p+2}^{(k)}(s,z,u) M_{J}(\od s,\od z,\od u), \quad t\in [0,T] ,\; k=1,\ldots, m.
	\end{align*}
\end{theo}

The proof  will be provided in \cref{sec:limit_proof} and \cref{sec:Skorokhod-space} contains some background information on weak convergence in the Skorokhod topology.

\begin{rema}
As a consequence of  \cref{thm:limit},  $(M_D^{\Pi_n},M_{B^{(1)}}^{\Pi_n},\ldots, M_{B^{(p)}}^{\Pi_n}, M_J^{\Pi_n})$ vaguely converges to $(M_D,M_{B^{(1)}},\ldots, M_{B^{(p)}}, M_J)$ in the following sense: For any $m\in \bN$,  and for any continuous functions with compact support $f_{l}^{(k)} \colon [0,T] \times [0,1]^d\to \bR$ ($l=0,\ldots,p$; $k=1,\ldots,m$),   $f_{p+2}^{(k)} \colon [0,T]  \times \bR^q_0 \times [0,1]^d\to \bR$ ($k=1,\ldots,m$),
the sequence of $\bR^m$-valued processes $\cX^n=(\cX^{n,(1)},\ldots, \cX^{n,(m)} )$  defined via
\begin{align*}
	\cX^{n,(k)}_t &= \int_{(0,t]\times[0,1]^d} f_{0}^{(k)}(s,u) M_D^{\Pi_n}(\od s,\od u)+ \sum_{l=1}^p \int_{(0,t]\times[0,1]^d} f_{l}^{(k)}(s,u) M_{B^{(l)}}^{\Pi_n}(\od s,\od u)  \\  
	& \quad +  	\int_{(0,t]\times \bR^q_0 \times[0,1]^d} f_{p+2}^{(k)}(s,z,u) M_{J}^{\Pi_n}(\od s,\od z,\od u), \quad t\in [0,T] ,\; k=1,\ldots, m,
\end{align*}
weakly converges in the Skorokhod topology on $\bD_T(\bR^m)$ to $\cX=(\cX^{(1)},\ldots, \cX^{(m)})$, where
\begin{align*}
	\cX^{(k)}_t &= \int_{(0,t]\times[0,1]^d} f_{0}^{(k)}(s,u) M_D(\od s,\od u)+ \sum_{l=1}^p \int_{(0,t]\times[0,1]^d}f_{l}^{(k)}(s,u) M_{B^{(l)}}(\od s,\od u)  \\  
	& \quad +  	\int_{(0,t]\times\bR^q_0 \times[0,1]^d} f_{p+2}^{(k)}(s,z,u) M_{J}(\od s,\od z,\od u), \quad t\in [0,T] ,\; k=1,\ldots, m.
\end{align*}
Indeed, if the $f_{l}^{(k)}$'s $(l=p+1,p+2;\;k=1,\ldots,m)$ in  \cref{thm:limit} have compact support, then there is an $\varepsilon>0$ (independent of $k,l,t,u$) such that $f_{l}^{(k)}=0$, if $0<|z|\leq \varepsilon$. Hence, we can apply \cref{thm:limit} with $R=\varepsilon$. We also refer to \cite{Ka17} for background information on the general theory of vague convergence of random measures and to \cite{Xi94,Xi95} for the case of martingale measures.
\end{rema}

In view of \cref{thm:limit}, the random measure formulation \eqref{eq:SDE-random_measure_Pi} of the grid-sampling SDE \eqref{eq:SDE-Pi_sampling} in  \cref{sec:discrete_sample_random_measure}, and the definition of $M_D$, a natural limit formulation of the grid-sampling SDE for a given randomized policy  $\bfh \colon [0,T] \times \bR^m  \times [0,1]^d \rightarrow \bR^{d}$ is
\begin{align}\label{eq:SDE-random_measure_limit}
	 X^{\bfh}_t = x &+ \int_{0}^t \nq \int_{[0,1]^d} b(s, X^{\bfh}_{s-}, {\bfh}(t,X^{\bfh}_{s-},u )) \od u \od s\notag \\ 
	 &+ \sum_{l=1}^p \int_{(0,t]\times[0,1]^d}  a^{(\cdot,l)}(s, X^{\bfh}_{s-},  {\bfh}(s,X^{\bfh}_{s-},u )) M_{B^{(l)}}(\od s,\od u)  \notag\\ 
	 &+  	\int_{(0,t]\times \{ 0<|z|\le \mathfrak{r}\} \times[0,1]^d} \gamI(s, X^{\bfh}_{s-},  {\bfh}(t,X^{\bfh}_{s-},u ), z)\tilde M_{J}(\od s,\od z,\od u) \notag \\ 
	 & +  	\int_{(0,t]\times \{ |z|> \mathfrak{r}\} \times[0,1]^d} \gamII(s, X^{\bfh}_{s-},  {\bfh}(t,X^{\bfh}_{s-},u ),z) M_{J}(\od s,\od z,\od u).
\end{align}
We call this SDE the \emph{grid-sampling limit SDE} for policy ${\bfh}$.

\begin{rema}
	We stress that the random measures $(M_D,M_{B^{(1)}},\ldots, M_{B^{(p)}} , M_J)$ appearing in the limit are independent, whereas  the pre-limit random measures $(M_D^{\Pi},M_{B^{(1)}}^{\Pi},\ldots, M_{B^{(p)}}^{\Pi} , M_J^{\Pi})$ are jointly constructed in terms of the randomization process $\xi^\Pi$ and are, thus, dependent. In particular, a solution $ X^{\bfh}$ of the grid-sampling limit SDE \eqref{eq:SDE-random_measure_limit} cannot be interpreted as the model dynamics evaluated along a $(\xi_t)_{t\in[0,T]}$-randomized policy, i.e., it cannot be reformulated in the form \eqref{eq:SDE-random} for some randomization process $\xi$ in general. Nonetheless, we think that the limit SDE \eqref{eq:SDE-random_measure_limit} is practically relevant for justifying learning algorithms derived by the first-optimize-then-discretize approach. This aspect will be briefly sketched in  \cref{sec:learning} below. 
\end{rema}

\begin{rema}\label{rem:no_jump}
	Suppose that we are in the no-jump case, i.e, $\gamI\equiv 0$.
	\begin{enumerate}
		\item Pathwise existence and uniqueness of the grid-sampling limit SDE \eqref{eq:SDE-random_measure_limit} follows from Proposition IV-1 in \cite{KM90}, provided the coefficients 
		\begin{equation*}
		b_{\bfh}(t,x,u)= b(t, x, {\bfh}(t,x,u )),\quad a_{\bfh}(t,x,u)= a(t, x, {\bfh}(t,x,u ))
		\end{equation*}
		are Lipschitz continuous and of linear growth in $x$ uniformly in $(t,u)$. Moreover, under these conditions, the law of $ X^{\bfh}$ is the unique solution of the martingale problem for the operator
		$$
		(\cL_{\bfh} f)(t,x):=\int_{[0,1]^d} \bb(\frac{1}{2} \sum_{i,j=1}^m (a_{\bfh}(t,x,u)a^\tran_{\bfh}(t,x,u))^{(i,j)} \frac{\partial^2f}{\partial x_i\partial x_j}(x)+  \sum_{i=1}^m (b_{\bfh}(t,x,u))^{(i)} \frac{\partial f}{\partial x_i}(x)\bb) \od u.
		$$
		\item By combining  \cref{thm:limit} with the stability results for SDEs driven by continuous orthogonal martingale measures in \cite[p.354]{KD01}, we observe that under at most technical assumptions the following limit theorem is valid: If $\bfh_1,\ldots,\bfh_K \colon [0,T]\times \bR^m\times [0,1]^d\rightarrow \bR^d$ are randomized policies, then one obtains the joint weak convergence $$(X^{\Pi_n,\bfh_1},\ldots, X^{\Pi_n,\bfh_K}, M_D^{\Pi_n},M_{B^{(1)}}^{\Pi_n},\ldots, M_{B^{(p)}}^{\Pi_n}) \to   (X^{\bfh_1},\ldots, X^{\bfh_K}, M_D,M_{B^{(1)}},\ldots, M_{B^{(p)}}).$$ 
		This result serves as another justification for using the grid-sampling limit SDE  \eqref{eq:SDE-random_measure_limit}.
	\end{enumerate}
We leave a detailed study of these aspects in the general case with jumps to future research.
\end{rema}

We close this subsection by two examples in which the grid-sampling limit SDE \eqref{eq:SDE-random_measure_limit} is simplified. They rely on the following elementary lemma, whose proof is given in \cref{sec:proofs-random-measure}.

\begin{lemm}\label{lem:Bm_from_white_noise}
	Suppose that $\eta \colon \Omega\times[0,T]\times [0,1]^d\rightarrow \bR^m$ is an $\bF$-predictable random field satisfying
	$$
		 \int_{[0,1]^d} \eta_t(u)\eta_t(u)^\tran\od u=I_m\quad \bP\otimes \Leb_{[0,T]}\trm{-a.e. } (\omega, t) \in \Omega \times [0, T].
	$$
	Define 
	$$
	B^{\eta,(k,l)}_t=\int_0^t \nq \int_{[0,1]^d} \eta^{(k)}_s(u)\, M_{B^{(l)}}(\od s,\od u),\quad t\in [0,T],\; l=1,\ldots,p,\,k=1,\ldots, m.
	$$
	Then, $B^\eta=(B^{\eta,(k,l)} : l=1,\ldots,p,\,k=1,\ldots, m)$ is an $mp$-dimensional Brownian motion.  
\end{lemm}

\begin{exam}\label{exmp:non_randomized}
Suppose that $\bfh$ is a classical, non-randomized control in feedback form, i.e., $\bfh$ does not depend on $u$. By  \cref{lem:Bm_from_white_noise} (with $\eta$ being the $\bR$-valued function which is constant 1),
$$
B_t^{\bf 1} =\bb( \int_0^t \nq \int_{[0,1]^d} \, M_{B^{(1)}}(\od s,\od u),\ldots, \int_0^t \nq \int_{[0,1]^d} \, M_{B^{(p)}}(\od s,\od u) \bb)^\tran
$$
is a $p$-dimensional Brownian motion. Moreover,
$$
N^{\bf 1}(\od t,\od z)= \int_{[0,1]^d} M_J(\od t,\od z, \od u)
$$
is a Poisson random measure  independent of $B^\eta$ with intensity $\nu_t(\od z) \od t$. Then, SDE \eqref{eq:SDE-random_measure_limit} can be re-written as
\begin{align*}
	X^{\bfh}_t  = x &+ \int_{0}^t  b(s, X^{\bfh}_{s-}, {\bfh}(s,X^{\bfh}_{s-} )) \od s\notag + \int_0^t a(s, X^{\bfh}_{s-},  {\bfh}(s,X^{\bfh}_{s-})) \od B^{\bf 1}_s  \\ 
	&+  	\int_{(0,t]\times \{ 0<|z|\le \mathfrak{r}\}} \gamI(s, X^{\bfh}_{s-},  {\bfh}(s,X^{\bfh}_{s-}), z)\tilde N^{\bf 1}(\od s,\od z,\od u) \\ &+  	\int_{(0,t]\times \{ |z|> \mathfrak{r}\}} \gamII(s, X^{\bfh}_{s-},  {\bfh}(s,X^{\bfh}_{s-}),z) N^{\bf 1}(\od s,\od z,\od u),
\end{align*}
i.e., we recover the dynamics \eqref{eq:SDE-classical}, as it should be.
\end{exam}

\begin{exam}\label{exmp:linear_control}
We now assume the drift coefficient $b$ and the diffusion coefficient $a$ are affine-linear in the control, i.e.,
$$
a(t,x,y)=a_0(t,x)+\sum_{j=1}^d y^{(j)}a_j(t,x),\quad b(t,x,y)=b_0(t,x)+\sum_{j=1}^d y^{(j)}b_j(t,x)
$$
for measurable functions $a_j \colon [0,T] \times \bR^m \rightarrow \bR^{m\times p}$ and $b_j \colon [0,T] \times \bR^m \rightarrow \bR^{m}$. The randomized control is given in terms of the measurable function $\bfh \colon [0,T] \times \bR^m  \times [0,1]^d \rightarrow \bR^{d}$. We assume that the coefficients are sufficiently regular to guarantee that a solution $X^{\bfh}$ to \eqref{eq:SDE-random_measure_limit} exists. Supposing that  $\bfh$ is square integrable with respect to the uniform distribution in the $u$-variable, we then consider the mean vector and covariance matrix 
$$
\mu_{\bfh}(t,x)=\int_{[0,1]^d} \bfh(t,x,u)\od u,\quad \Theta_{\bfh}(t,x)= \int_{[0,1]^d} (\bfh(t,x,u)-\mu_{\bfh}(t,x))(\bfh(t,x,u)-\mu_{\bfh}(t,x))^\tran \od u
$$
as a function of $(t,x)$. Assuming that $\Theta_{\bfh}(t,x)$ is  positive definite for every $(t,x)\in [0,T] \times \bR^m$, we write $\vartheta_{\bfh}(t,x)$ for  the positive definite matrix root of $\Theta_{\bfh}(t,x)$ and define
$$
\eta_{\bfh} \colon [0,T]\times \bR^m\times[0,1]^d\rightarrow \bR^d, (t,x,u)\mapsto \vartheta_{\bfh}(t,x)^{-1}(\bfh(t,x,u)-\mu_{\bfh}(t,x)).
$$
Note that for every $(t,x)\in [0,T]\times \bR^m$
$$
\int_{[0,1]^d} \eta_{\bfh}(t,x,u) \od u=0,\quad \int_{[0,1]^d} \eta_{\bfh}\eta_{\bfh}^\tran(t,x,u) \od u=I_d.
$$
Thus, the $\bR^{d+1}$-valued random field
$$
\eta_t(u)=(\eta^{(1)}_{\bfh}(t,X^{\bfh}_{t-},u),\ldots, \eta^{(d)}_{\bfh}(t,X^{\bfh}_{t-},u),1)^\tran
$$
satisfies the assumptions of \cref{lem:Bm_from_white_noise} and we denote the corresponding Brownian motion by $B^\eta=(B^{\eta,(i,l)})_{i=1,\ldots,d+1,\;l=1,\ldots,p}$. Then, the white noise measures can be replaced by the $(d+1)p$-dimensional Brownian motion $B^\eta$ and \eqref{eq:SDE-random_measure_limit} becomes
\begin{align*}
	 X^{\bfh}_t  = x &+ \int_{0}^t\bb( b_0(s, X^{\bfh}_{s-})+\sum_{j=1}^d   b_j(s, X^{\bfh}_{s-})\mu_{\bfh}^{(j)}(s,X^{\bfh}_{s-}) \bb)\od s\\ &+\sum_{l=1}^p \int_{0}^t\bb( a^{(\cdot,l)}_0(s, X^{\bfh}_{s-})+\sum_{j=1}^d   a^{(\cdot,l)}_j(s, X^{\bfh}_{s-})\mu_{\bfh}^{(j)}(s,X^{\bfh}_{s-}) \bb)\od B^{\eta,(d+1,l)}_s \\ &+\sum_{l=1}^p\sum_{i=1}^d \int_0^t  \bb(\sum_{j=1}^d  a^{(\cdot,l)}_j(s, X^{\bfh}_{s-}) \vartheta_{\bfh}^{(j,i)}(s,X^{\bfh}_{s-})\bb)\od B^{\eta,(i,l)}_s   \notag\\ 
	 &+  	\int_{(0,t]\times \{ 0<|z|\le \mathfrak{r}\} \times[0,1]^d} \gamI(s, X^{\bfh}_{s-},  {\bfh}(s,X^{\bfh}_{s-},u ), z)\tilde M_{J}(\od s,\od z,\od u) \notag \\ 
	 &+  	\int_{(0,t]\times \{ |z|> \mathfrak{r}\} \times[0,1]^d} \gamII(s, X^{\bfh}_{s-},  {\bfh}(s,X^{\bfh}_{s-},u ),z) M_{J}(\od s,\od z,\od u).
\end{align*}
This example extends the analogous SDE formulation for entropy-regularized mean-variance portfolio optimization with jumps derived in \cite{BN23}. Note, however,  that the white noise measure approach clarifies that (and how exactly) the driving Brownian motion depends on the choice of the randomized control ${\bfh}$.
\end{exam}

\subsection{Comparison to the exploratory SDE of \cite{WZZ20}}\label{sec:exploratory}

In this subsection, we briefly compare the grid-sampling limit SDE \eqref{eq:SDE-random_measure_limit} to the exploratory SDE introduced in \cite{WZZ20}. In order to keep the notation simple, we confine ourselves to the one-dimensional case ($m=p=d=1$) without jumps $\gamI=0$, compare \cite{WZZ20}. We note, however, that the multivariate case of the exploratory SDE is covered in \cite{JZ22} and, recently, a setting with jumps has been developed in  \cite{GLZ24}. In any of these cases, the derivation of the exploratory SDE relies on a heuristic law of large number argument to extract the semimartingale characteristics when averaging over independent executions of a relaxed control.

Given a relaxed control  ${h} \colon [0,T]\times \bR \to \mathcal{P}r(\cB(\bR))$ with Lebesgue density $\dot {h}(t,x,\cdot)$, the exploratory SDE takes the form
\begin{equation*}
	\tilde X^{h}_t=	x+\int_0^t \nq \int_{\bR} b(s,\tilde X^{h}_s,y)\dot {h}(s,\tilde X^{h}_s,y) \od y\od s+ \int_0^t \sqrt{\int_{\bR} a(s,\tilde X^{h}_s,y)^2\dot {h}(s,\tilde X^{h}_s,y) \od y}\,\od W_s
\end{equation*}
for some 1-dimensional Brownian motion $W$. Lemma 2 in \cite{JZ22} states sufficient conditions on $b$, $a$, and $\dot h$ for existence and uniqueness of a strong solution. Note that the law of $\tilde X^h$ then solves the martingale problem for the operator
 	\begin{equation}\label{eq:generator_simplified}
 (\cL_{h} f)(t,x):=\int_{\bR} \bb(\frac{1}{2} a(t,x,y)^2 f''(x)+  b(t,x,y) f'(x)\bb) \dot {h}(t,x,y)\od y.
 \end{equation}
We now assume that ${\bfh}$ is a randomized control, which executes $h$, and that the assumptions of \cref{rem:no_jump}(1) are satisfied. By a change of variables, the law of the unique solution $X^{\bfh}$ to the grid-sampling limit SDE solves the martingale problem for the same operator $\cL_{h}$ and, by uniqueness of the martingale problem under the Lipschitz assumptions, $\tilde X^h$ and $X^{\bfh}$ have the same probability law. Hence, in a stochastic control framework (e.g., to compute the expected cost of a given relaxed/randomized control pair $h$, ${\bfh}$ or for the derivation of an HJB equation), the grid-sampling limit SDE $X^{\bfh}$ and the exploratory SDE $\tilde X^h$ will lead to the same result -- and it is a matter of taste which one to use. In the first SDE the white noise martingale measure comes up, while,  in the second SDE, one has to deal with the square-root in the diffusion coefficient, compare the Remarks in \cite[pp. 350--351]{KD01}.

However, if one considers several controls at the same time, the joint distribution of $(\tilde X^{h_1}, \tilde X^{h_2})$ and $(X^{\bfh_1}, X^{\bfh_2})$ may differ, as illustrated by the following simple example.
\begin{exam}\label{exmp:two_controls}
	Suppose $T=1$, $b=0$ and $a(t,x,u)=u$. We apply the randomized controls ${\bfh}_j(t,x,u)=\mu_j+\sigma_j \Phi^{-1}(u)$, ($\mu_j\in \bR$, $\sigma_j>0$, $j=1,2$), which execute a Gaussian law $h_j(t,x)$ with mean $\mu_j$ and variance $\sigma_j^2$ independent of the time and state of the system. For a fixed sampling partition $\Pi$, the predictable covariation of the model dynamics along the $\xi^\Pi$-randomized controls satisfies
	$$
	\langle X^{\Pi,\bfh_1}, X^{\Pi,\bfh_2} \rangle_1=\sum_{i=1}^n (t_i-t_{i-1}) (\mu_1+\sigma_1 \Phi^{-1}(\xi^\Pi_{t_i}))(\mu_2+\sigma_2 \Phi^{-1}(\xi^\Pi_{t_i}))
	$$  
	If, e.g., $\Pi_n$ is the equidistant partition of the unit interval into $n$ subintervals, then a straightforward application of the strong law of large numbers implies, a.s.,
	$$
	\langle X^{\Pi_n,\bfh_1}, X^{\Pi_n,\bfh_2} \rangle_1 \rightarrow \bE\left[(\mu_1+\sigma_1 \Phi^{-1}(\xi^\Pi_{t_1}))(\mu_2+\sigma_2 \Phi^{-1}(\xi^\Pi_{t_1}))\right]=\mu_1\mu_2+\sigma_1\sigma_2.
	$$
	This limit coincides with the predictable covariation of the grid-sampling limit SDEs, because, by Proposition I-6(2) in \cite{KM90},
	\begin{align*}
	\langle X^{\bfh_1}, X^{\bfh_2} \rangle_1 & = \bb\langle \int_{(0,\cdot]\times[0,1]^d} (\mu_1+\sigma_1  \Phi^{-1}(u)) M_{B}(\od s,\od u),   \int_{(0,\cdot]\times[0,1]^d} (\mu_2+\sigma_2  \Phi^{-1}(u)) M_{B}(\od s,\od u) \bb\rangle_1 \\   
	& = \int_{(0,1]\times[0,1]^d}  (\mu_1+\sigma_1  \Phi^{-1}(u)) (\mu_1+\sigma_1  \Phi^{-1}(u)) \od s \od u=\mu_1\mu_2+\sigma_1\sigma_2.
	\end{align*}
However, the predictable covariation of the corresponding exploratory SDE is
\begin{align*}
&	\langle  \tilde X^{h_1}, \tilde X^{h_2} \rangle_1\\
&=\BB\langle  \int_0^\cdot \sqrt{\int_{\bR} y^2 \frac{1}{\sqrt{2\pi \sigma_1^2}} \e^{-(y-\mu_1)^2/(2\sigma_1^2)}\od y}\,\od W_s, \int_0^\cdot \sqrt{\int_{\bR} y^2 \frac{1}{\sqrt{2\pi \sigma_2^2}} \e^{-(y-\mu_2)^2/(2\sigma_2^2)}\od y}\,\od W_s\BB\rangle_1 \\
 &= \sqrt{(\mu_1^2+\sigma_1^2)(\mu_2^2+\sigma_2^2)}.
	\end{align*}
\end{exam}

Let us summarize: By the considerations at the beginning of this subsection $\tilde X^h$ and $X^{\bfh}$ have the same probability law, if ${\bfh}$ executes $h$. The SDEs governing these two processes cannot be interpreted as dynamics of the system along a $\xi$-randomized control. One way to justify these SDEs is to view them as the limit dynamics of the grid-sampling SDE, which has a sound interpretation in terms of $\xi^\Pi$-randomized controls. By \cref{rem:no_jump}(2), we observe that the law of  $X^{\Pi_n,\bfh}$ converges to the law of $\tilde X^h$ under at most technical conditions for one fixed control pair $h$, ${\bfh}$. However, as illustrated by  \cref{exmp:two_controls}, one cannot hope that the joint convergence result
to the grid-sampling limit SDEs indicated in \cref{rem:no_jump}(2) carries over
 to the exploratory SDE.  We will illustrate in the next subsection that this difference can be essential for the justification of learning algorithms.

\subsection{Outlook: Towards learning}\label{sec:learning}

In this subsection, we exemplify how the algorithms of the first-optimize-then-discretize approach of \cite{JZ22a,JZ22, JZ23} can be justified by applying the grid-sampling limit SDE \eqref{eq:SDE-random_measure_limit} instead of the sample SDE of \cite{JZ22, JZ23}. In this way we can ensure that the derivation bypasses any potential problems related to idealized sampling.

For sake of illustration, we will here only consider the problem of policy evaluation of a fixed randomized control $\bfh$ and restrict ourselves to the no-jump case in dimension one ($m=d=p=1$). Assuming that the Lipschitz conditions in \cref{rem:no_jump} are satisfied, the unique solution of the  grid-sampling limit SDE takes the form
\begin{eqnarray*}
	X^{\bfh}_t =x+ \int_{0}^t \nq \int_0^1 b(s, X^{\bfh}_{s}, {\bfh}(s,X^{\bfh}_{s},u )) \od u \od s +\int_{(0,t]\times[0,1]}  a(s, X^{\bfh}_{s},  {\bfh}(s,X^{\bfh}_{s},u )) M_{B}(\od s,\od u). 
\end{eqnarray*}
We suppose that the law of $\bfh(t,x,\eta)$ (where $\eta$ is a uniform random variable on $[0,1]$) is absolutely continuous with respect to the Lebesgue measure with density $\dot h(t,x,\cdot)$ for every $(t,x)\in [0,T]\times \bR$ and that its Shannon entropy
$$
-\int_{\bR} \dot h(t,x,y) \log\dot h(t,x,y)\od y
$$ 
exists in $\bR$ and is measurable and bounded as a function in $(t,x)$. We consider the problem of evaluating the expected terminal cost with a running entropy-regularization term, which rewards exploration, as suggested in \cite{WZZ20}. The corresponding cost process is given by
$$
\cJ^{\bfh}_t=\bE\bb[g(X^{\bfh}_T) +\lambda \int_t^T \nq \int_{\bR} \dot h(s,X^{\bfh}_s,y)\log\dot h(s,X^{\bfh}_s,y)\od y \od s \,\bb|\, \cF_t \bb]
$$
for some fixed temperature parameter $\lambda>0$. We here assume, for the sake of simplicity, that the terminal cost function $g$ is bounded, and, consequently, the process $\cJ^{\bfh}$ is bounded as well. We say that a measurable function $J^\bfh \colon [0,T] \times \bR \rightarrow \bR$ is a version of the \emph{value function} of ${\bfh}$, if 
$$
J^\bfh(t,X^{\bfh}_t)=\cJ^{\bfh}_t\quad \bP\textnormal{-a.s.,} \quad t \in [0, T].
$$
The aim of policy evaluation is to learn the value function  $J^\bfh$ from observations of the system $X^{\xi,\bfh}$, when feeding in the $\xi$-randomized policy $\bfh(t,x,\xi_t)$ for some randomization process $\xi$, without knowing the true model parameters $b,a$. Recall that in the simplified setting of this subsection
\begin{equation}\label{eq:randomized_nojumps}
	\od  X^{\xi,\bfh}_t  = b(t,  X^{\xi,\bfh}_{t}, {\bfh}(t, X^{\xi,\bfh}_{t},\xi_t )) \od t + a(t,  X^{\xi,\bfh}_{t},  {\bfh}(t, X^{\xi,\bfh}_{t},\xi_t )) \od B_t,\quad X^{\xi,\bfh}_0=x.
\end{equation}
The algorithms for policy evaluation derived in \cite{JZ22a,JZ22} rely on the martingale characterization of the value function $J^\bfh$, which can be formulated for the grid-sampling limit SDE in the following way (see  \cref{app:proof-value} for the routine proof).
\begin{prop}\label{prop:value}
 \begin{enumerate}[\quad \rm(1)]
	\item  	Suppose that the following partial differential equation has a bounded solution $J\in C^{1,2}([0,T]\times \bR)$:
	$$
	\frac{\partial J}{\partial t}(t,x)+(\cL_{h} J(t,\cdot))(t,x)+\lambda \int_{\bR} \dot h(t,x,y) \log\dot h(t,x,y)\od y=0,\quad (t,x)\in[0,T)\times \bR,
	$$
	with the terminal condition $J(T,\cdot)=g$ (where the differential operator $\cL_{h}$ is defined in \eqref{eq:generator_simplified}). Then, $J$ is a version of the value function of ${\bfh}$.
	\item Assume that $\tilde J \colon [0,T]\times\bR\rightarrow \bR$ is measurable with $\tilde J(T,\cdot)=g$. Then, $\tilde J$ is a version of the value function of ${\bfh}$, if and only if
	$$
	\tilde J(t, 	X^{\bfh}_t)+ \lambda \int_0^t \nq \int_{\bR} \dot h(s,X^{\bfh}_s,y)\log\dot h(s,X^{\bfh}_s,y) \od y \od s, \quad 0\leq t \leq T,
	$$
	is an $\bF$-martingale.
\end{enumerate}
\end{prop}
We now provide an alternative derivation of the offline variant of the continuous-time TD(0)-algorithm in \cite{JZ22a,JZ22}: To this end, fix a parametric class of functions $\{J_\vartheta:\vartheta \in \Theta\}$ for some open parameter set $\Theta \subseteq \bR^L$. We will implicitly assume that the function
$$
{\bf J}_{\Theta} \colon [0,T]\times \bR\times \Theta\rightarrow \bR,\quad (t,x,\vartheta)\mapsto J_\vartheta(t,x)
$$
satisfies sufficient smoothness and boundedness assumptions to justify the manipulations below. Moreover, we postulate that $J_\vartheta(T,\cdot)=g$ for every $\vartheta\in \Theta$.  We aim at finding a parameter $\vartheta^*\in \Theta$ such that  $J_{\vartheta^*}$ is a good approximation to the value function $J^{\bfh}$ of the randomized control $\bfh$. Since integrals of sufficiently good integrands with respect to  a martingale have zero expectation, the martingale characterization of the value function in \cref{prop:value} motivates to search for a parameter $\vartheta^*$ such that
$$
\bE\bb[\int_0^T \nabla_\vartheta {\bf J}_{\Theta}(s, X^{\bfh}_s,\vartheta^*)\bb(\od J_{\vartheta^*}(s, X^{\bfh}_s,\vartheta^*)+\lambda  \int_{\bR} \dot h(s,X^{\bfh}_s,y) \log\dot h(s,X^{\bfh}_s,y)
 \od y\od s\bb) \bb] =0,
$$
compare \cite{JZ22a}.
Here, $\nabla_\vartheta$ stands for the gradient in the $\vartheta$-variable. Then, stochastic approximation \cite{RM51} suggests to consider the update step
\begin{equation}\label{eq:TD_update}
	\vartheta \leftarrow \alpha \int_0^T \nabla_\vartheta {\bf J}_{\Theta}(s, X^{\bfh}_s,\vartheta)\bb(\od J_{\vartheta}(s, X^{\bfh}_s)+\lambda  \int_{\bR} \dot h(s,X^{\bfh}_s,y) \log\dot h(s,X^{\bfh}_s,y)
	 \od y\od s\bb)
\end{equation}
for some step-size $\alpha>0$. 
Up to here, the derivation follows exactly the one in \cite{JZ22a,JZ22} with the grid-sampling limit SDE in place of the sample SDE of \cite{JZ22}. Note that, although the unknown coefficients $b$ and $a$ do not show up in \eqref{eq:TD_update}, its implementation is infeasible, because $X^{\bfh}$ is not observable (it is not the  response of the system to a $\xi$-randomized control). We view \eqref{eq:TD_update} as an idealized continuous-limit update step, which will be discretized next. By It\^o's formula, recalling that $\cL_h$ in \eqref{eq:generator_simplified} is the infinitesimal generator of $X^{\bfh}$, we obtain
\begin{align}\label{eq:TD_derivation}
& \int_0^T \nabla_\vartheta {\bf J}_{\Theta}(s, X^{\bfh}_s,\vartheta)\bb(\od J_{\vartheta}(s, X^{\bfh}_s)+\lambda  \int_{\bR} \dot h(s,X^{\bfh}_s,y) \log\dot h(s,X^{\bfh}_s,y)
 \od y \od s \bb)\notag\\
 &=\int_0^T  \nabla_\vartheta {\bf J}_{\Theta}(s, X^{\bfh}_s,\vartheta) \frac{\partial J_{\vartheta}}{\partial t}(s, X^{\bfh}_s)\od s\notag \\ 
 & \quad + \int_0^T\nabla_\vartheta {\bf J}_{\Theta}(s, X^{\bfh}_s,\vartheta)\bb((\cL_h J_{\vartheta}(s,\cdot))(s, X^{\bfh}_s)+\lambda  \int_{\bR} \dot h(s,X^{\bfh}_s,y) \log\dot h(s,X^{\bfh}_s,y)
 \od y\bb)\od s\notag \\
 & \quad + \int_{(0, T]\times[0,1]} \nabla_\vartheta {\bf J}_{\Theta}(s, X^{\bfh}_s,\vartheta) a(s, X^{\bfh}_{s},  {\bfh}(s,X^{\bfh}_{s},u )) \frac{\partial J_{\vartheta}}{\partial x}(s, X^{\bfh}_s) M_{B}(\od s,\od u).
\end{align}
By change of variables and applying the notation introduced in \cref{rem:no_jump}(1), the second integral on the right-hand side of \eqref{eq:TD_derivation} becomes
\begin{align*}
	\int_0^T \nq \int_0^1 \nabla_\vartheta {\bf J}_{\Theta}(s, X^{\bfh}_s,\vartheta)
	&\bb(\frac{1}{2} a_{\bfh}(s, X^{\bfh}_s,u)^2 \frac{\partial^2 J_{\vartheta}}{\partial x^2}(s, X^{\bfh}_s)+  b_{\bfh}(s, X^{\bfh}_s,u) \frac{\partial J_{\vartheta}}{\partial x}(s, X^{\bfh}_s)\\ &\quad+\lambda  \log\dot h(s,X^{\bfh}_s,{\bfh}(s, X^{\bfh}_s,u))
	\bb) \od u \od s,
\end{align*}
which, in fact, is an integral with respect to the limit drift measure $M_D$. Thus, the joint convergence in \cref{rem:no_jump}(2) suggests that
$$
\int_0^T \nabla_\vartheta {\bf J}_{\Theta}(s, X^{\bfh}_s,\vartheta)\bb(\od J_{\vartheta}(s, X^{\bfh}_s)+\lambda  \int_{\bR}\dot h(s,X^{\bfh}_s,y)  \log\dot h(s,X^{\bfh}_s,y)
\od u\od s \bb)
$$
can be approximated in law by
\begin{align*}
&\int_{(0,T]\times [0,1]}  \nabla_\vartheta {\bf J}_{\Theta}(s, X^{\Pi,\bfh}_s,\vartheta)\bb(\frac{\partial J_{\vartheta}}{\partial t}(s, X^{\Pi,\bfh}_s)+\lambda  \log\dot h(s,X^{\Pi,\bfh}_s,{\bfh}(s, X^{\Pi,\bfh}_s,u)) \notag \\
&\hspace{2cm}+\frac{1}{2}a_{\bfh}(s, X^{\Pi,\bfh}_s,u)^2 \frac{\partial^2 J_{\vartheta}}{\partial x^2}(s, X^{\Pi,\bfh}_s)+  b_{\bfh}(s, X^{\Pi,\bfh}_s,u) \frac{\partial J_{\vartheta}}{\partial x}(s, X^{\Pi,\bfh}_s)
\bb) M_D^\Pi(\od s,\od u) \notag \\ 
& + \int_{(0,T]\times[0,1]} \nabla_\vartheta {\bf J}_{\Theta}(s, X^{\Pi,\bfh}_s,\vartheta) a_{\bfh}(s, X^{\Pi,\bfh}_{s},  u) \frac{\partial J_{\vartheta}}{\partial x}(s, X^{\Pi,\bfh}_s) M^\Pi_{B}(\od s,\od u)
\end{align*}
for a sufficiently fine sampling grid $\Pi$, where $X^{\Pi,\bfh}$ solves the SDE \eqref{eq:randomized_nojumps} with $\xi=\xi^\Pi$. In view of \cref{lem:drift_Pi,lem:Bm_Pi}, and applying It\^o's formula once more, this expression equals
$$
\sum_{i=1}^n \int_{t_{i-1}}^{t_i} \nabla_\vartheta {\bf J}_{\Theta}(s, X^{\Pi,\bfh}_s,\vartheta)\B(\od J_{\vartheta}(s, X^{\Pi,\bfh}_s)+\lambda   \log\dot h(s,X^{\Pi,\bfh}_s,{\bfh}(s, X^{\Pi,\bfh}_s,\xi^{\Pi}_{t_i}))
\od s \B),
$$
leading to the modified update step
\begin{equation}\label{eq:TD_update_Pi}
	\vartheta \leftarrow \alpha\sum_{i=1}^n \int_{t_{i-1}}^{t_i} \nabla_\vartheta {\bf J}_{\Theta}(s, X^{\Pi,\bfh}_s,\vartheta)\B(\od J_{\vartheta}(s, X^{\Pi,\bfh}_s)+\lambda   \log\dot h(s,X^{\Pi,\bfh}_s,{\bfh}(s, X^{\Pi,\bfh}_s,\xi^{\Pi}_{t_i}))
	\od s \B). 
\end{equation}
Here, the $t_i$'s are, of course, the grid points of the sampling grid $\Pi$. We emphasize that the update step \eqref{eq:TD_update_Pi} is independent of the unknown parameters $\mu$ and $a$ and only depends on observables, namely the grid-sampling randomization process $\xi^\Pi$ and the response $X^{\Pi,\bfh}$ of the system to the $\xi^\Pi$-randomized policy $\bfh(t,x,\xi^\Pi_t)$. Note that the update-step \eqref{eq:TD_update_Pi} is still formulated in continuous time. For the actual implementation, it is natural to consider the time-discretization relative to $\Pi$ given by
\begin{align*}
	\vartheta \leftarrow \alpha\sum_{i=1}^n  \nabla_\vartheta {\bf J}_{\Theta}(t_{i-1}, X^{\Pi,\bfh,E}_{t_{i-1}},\vartheta)&\B[J_{\vartheta}(t_{i}, X^{\Pi,\bfh,E}_{t_{i}}) -J_{\vartheta}(t_{i-1}, X^{\Pi,\bfh,E}_{t_{i-1}})\\ &\quad +\lambda  ({t_{i}}-{t_{i-1}})  \log\dot h({t_{i-1}},X^{\Pi,\bfh,E}_{t_{i-1}},{\bfh}(s, X^{\Pi,\bfh,E}_{t_{i-1}},\xi^{\Pi}_{t_i}))
\B],
\end{align*}
where $X^{\Pi,\bfh,E}$ is the Euler approximation to $X^{\Pi,\bfh}$ relative to the grid $\Pi$. This expression coincides with the TD(0)-update step for policy evaluation in \cite{JZ22}, see, e.g., lines -12 and -8 in their Algorithm~4. Hence, we have provided a new justification of the continuous-time TD(0)-algorithm for policy evaluation, which  avoids making use of idealized sampling.


\section{Proof of \cref{thm:limit}} \label{sec:limit_proof}

\subsection{Preliminaries}\
To avoid double-indexing, we assume that $\Pi_n$ partitions $[0,T]$ into $n$ subintervals and write $0  = t^n_0 <\cdots < t^n_n =T$ for the grid points of $\Pi_n$. We emphasize that the same proof also works, even if $\Pi_n$ decomposes $[0,T]$ into $k(n)\in \bN$, which is not necessarily equal to $n$, subintervals.  Denote 
\begin{align*}
	\bfU: = [0, T] \times [0, 1]^d, \quad \bfV: = [0, T] \times \bR^q_0 \times [0, 1]^d.
\end{align*}
The assumptions imply that $f_l \in B_b(\bfU; \bR^m)$ for $l = 0, \ldots, p$ and $f_l \in B_b(\bfV; \bR^m)$ for $ l= p+1, p+2$. Moreover, by \cref{rema:levy_integration},
\begin{align}\label{assumption:Levy-measure}
	\int_0^T \nq \int_{\bR^q_0} ( |z|^2 \1_{\{0 < |z| \le R\}}  + \1_{\{|z| > R\}}) \nu_s(\od z)  \od s < \infty, \quad \forall R \in (0, \infty) \cup\{\frr\}.
\end{align}
In view of \cref{lem:drift_Pi,lem:Bm_Pi,lem:jumps_Pi},
we have the representation
\begin{align}\label{eq:decomposition:discrete-time-integral}
	\cX^n_t & = \sum_{i=1}^n  \bb[ \int_0^t f_0 (s, \xi^n_i) \1_{(t^n_{i-1}, t^n_i]}(s)  \od s   + \sum_{l=1}^p \int_0^t f_l (s, \xi^n_i) \1_{(t^n_{i-1}, t^n_i]}(s)  \od B^{(l)}_s \notag\\
	& \qquad + \int_0^t \nq \int_{0 < |z| \le R} f_{p+1}(s, z, \xi^n_i) \1_{(t^n_{i-1}, t^n_i]}(s) |z| \tilde N(\od s, \od z) \notag \\
	& \qquad  + \int_0^t \nq \int_{|z| > R} f_{p+2}(s, z, \xi^n_i) \1_{(t^n_{i-1}, t^n_i]}(s) N(\od s, \od z)\bb].
\end{align}

We will also consider the piecewise constant interpolation of $\cX^n$ between the grid points of $\Pi_n$. Introducing the notation
	\begin{align*}
		\rho_n(t) : = \sup\{t^n_i : t^n_i \le t\}, \quad t \in [0, T],
	\end{align*}
	it can be written as $\cX^n_{\rho_n(t)}$, $t\in [0,T]$. 
	
	By Theorem 3.1 in \cite {Bi99}, it suffices to show that, as $n \to \infty$, 
\begin{align}\label{converge-prob-error}
	\tilde d^m_T(\cX^n,  \cX^n_{\rho_n}) \xrightarrow{\bPn} 0,
\end{align}
and
\begin{align}\label{eq:weak-convergence-rho}
	\cX^n_{\rho_n} \xrightarrow{\scrD_T} \cX,
\end{align}
where the metric  $\tilde d^m_T$, which is defined in  \cref{sec:Skorokhod-space}, induces the Skorokhod topology on the space $\bD_T(\bR^m)$ of c\`adl\`ag functions $F\colon [0,T]\rightarrow \bR^m$ and $\xrightarrow{\mathscr D_T}$ stands for convergence in distribution in the Skorokhod space.
	The proof of assertions \eqref{converge-prob-error} and \eqref{eq:weak-convergence-rho} will be provided in \cref{sec:prob_proof} and \cref{sec:dist_proof}, respectively.

\subsection{Proof of assertion (\ref{converge-prob-error})}\label{sec:prob_proof}
Let $\kappa \in (0, \infty) \cap (0, R]$ and let $\kappa = 0$ if $R = 0$. We define the process $\cX^{n, \kappa}$ by setting
\begin{align*}
	\cX^{n, \kappa} : = \cX^n - \sum_{i=1}^n \int_0^{\cdot} \nq \int_{0 < |z| \le \kappa} f_{p+1}(s, z, \xi^n_i) \1_{(t^n_{i-1}, t^n_i]}(s) |z| \tilde N(\od s, \od z).
\end{align*}
By separating $\tilde N =  N - \nu$ on $[0, T] \times \{\kappa < |z| \le R\}$, which is possible as $\int_{0}^T \int_{\kappa < |z| \le R} |z| \nu_s(\od z) \od s < \infty$ and $f_{p+1}$ is bounded, and then rearranging terms we get
\begin{align*}
	\cX^{n, \kappa} & = \sum_{i=1}^n \int_0^{\cdot} f_0(s, \xi^n_i) \1_{(t^n_{i-1}, t^n_i]}(s) \od s   - \sum_{i=1}^n  \int_0^{\cdot} \nq \int_{\kappa < |z| \le R} f_{p+1}(s, z, \xi^n_i) \1_{(t^n_{i-1}, t^n_i]}(s) |z| \nu_s(\od z) \od s\\
	& \quad + \sum_{i=1}^n  \sum_{l=1}^p \int_0^{\cdot} f_l(s, \xi^n_i) \1_{(t^n_{i-1}, t^n_i]}(s) \od B^{(l)}_s \\
	& \quad + \sum_{i=1}^n \int_0^{\cdot}\nq \int_{|z| > \kappa} [f_{p+1}(s, z, \xi^n_i) |z| \1_{\{0 < |z| \le R\}} + f_{p+2}(s, z, \xi^n_i) \1_{\{|z| >R\}}] \1_{(t^n_{i-1}, t^n_i]}(s)  N(\od s, \od z)\\
	& = : (\cX^{n, \kappa}_D - \cX^{n, \kappa}_\nu + \cX^{n, \kappa}_B ) + \cX^{n, \kappa}_J\\
	& =: \cX^{n, \kappa}_C + \cX^{n, \kappa}_J.
\end{align*}
Using the triangle inequality we obtain
\begin{align}\label{eq:estimate-Skorokhod-Y-1}
	\tilde d^m_T(\cX^n,  \cX^n_{\rho_n}) & \le \tilde d^m_T(\cX^n, \cX^{n,\kappa}) + \tilde d^m_T( \cX^{n,\kappa},  \cX^{n,\kappa}_{\rho_n}) + \tilde d^m_T( \cX^{n,\kappa}_{\rho_n},  \cX^{n}_{\rho_n}) \notag \\
	& \le \sup_{t \in [0, T]} |\cX^n_t - \cX^{n, \kappa}_t| + \tilde d^m_T( \cX^{n,\kappa},  \cX^{n,\kappa}_{\rho_n}) + \sup_{t \in [0, T]} |\cX^{n,\kappa}_{\rho_n(t)} -  \cX^{n}_{\rho_n(t)}| \notag \\
	& \le 2 \sup_{t \in [0, T]} |\cX^n_t - \cX^{n, \kappa}_t| + \tilde d^m_T( \cX^{n,\kappa},  \cX^{n,\kappa}_{\rho_n}).
\end{align}
For $\ep >0$, since $\cX^n - \cX^{n, \kappa}$ is an $\bF^{\Pi_n}$-martingale, applying Doob's maximal inequality yields
\begin{align}\label{eq:estimate-Skorokhod-Y-2}
	\bPn \bb(\bb\{\sup_{t \in [0, T]} |\cX^n_t - \cX^{n, \kappa}_t| > \ep \bb\}\bb) & \le 4 \ep^{-2} \bEn\bb[\int_0^T \nq \int_{0 < |z| \le\kappa} \sum_{i=1}^n |f_{p+1}(s, z, \xi^n_i)|^2 \1_{(t^n_{i-1}, t^n_i]}(s) |z|^2 \nu_s(\od z) \od s \bb] \notag \\
	& \le 4 \ep^{-2} \|f_{p+1}\|_{B_b(\bfV; \bR^m)}^2 \int_0^T  \nq \int_{0 < |z| \le\kappa} |z|^2 \nu_s(\od z) \od s.
\end{align}
We now deal with the term $\tilde d^m_T( \cX^{n,\kappa},  \cX^{n,\kappa}_{\rho_n})$. Set $\tau^n_0 : = 0$ and
\begin{align*}
	\tau^n_i : = \inf\{t \in (t^n_{i-1}, t^n_i] : |\Delta L_t| > \kappa\} \wedge t^n_i, \quad i = 1, \ldots, n,
\end{align*}
with the convention $\inf \emptyset : = \infty$, and denote the events $A^{n, \kappa}_i$ by
\begin{align*}
	A^{n, \kappa}_i & : = \bb\{\int_{(t^n_{i-1}, t^n_i] \times \{|z| > \kappa\}} N(\od s, \od z) \le 1\bb\}, \quad i = 1, \ldots, n-1,\\
	A^{n, \kappa}_n & : = \bb\{\int_{(t^n_{n-1}, T] \times \{|z| > \kappa\}} N(\od s, \od z) = 0\bb\}.
\end{align*}
Then $t^n_{i-1} < \tau^n_i \le t^n_i$ on $A^{n, \kappa}_i$ and $\tau^n_n = T$ on $A^{n, \kappa}_n$. Now, for $\omega \in \cap_{i=1}^n A^{n, \kappa}_i$, we define the function $\lambda = \lambda_{\omega, n, \kappa} \colon [0, T] \to [0, T]$ which  piecewise linearly interpolates the points $(0,0), (\tau^n_1, t^n_1), \ldots$, $(\tau^n_{n-1}, t^n_{n-1})$, $(\tau^n_n, T)$. Namely,
\begin{align*}
	\lambda(t) = t^n_{i-1} + (t^n_i - t^n_{i-1})\frac{t - \tau^n_{i-1}}{\tau^n_i - \tau^n_{i-1}}, \quad t \in (\tau^n_{i-1}, \tau^n_i], \quad i = 1, \ldots, n.
\end{align*}
Then, $\lambda$ is a strictly increasing and continuous function with $\lambda(0) = 0$, $\lambda(T) = T$. It is clear that, for all $t \in (\tau^n_{i-1}, \tau^n_i]$, $i = 1, \ldots, n$,
\begin{align*}
	|\lambda (t) - t| \le \max\{t^n_{i-1} - \tau^n_{i-1}, \tau^n_i - t^n_{i-1}\} + (t^n_i - t^n_{i-1})\frac{t - \tau^n_{i-1}}{\tau^n_i - \tau^n_{i-1}} \le 2 |\Pi_n|.
\end{align*}
Hence, on $\cap_{i = 1}^n A^{n, \kappa}_i$ and for such a choice of $\lambda$ as above, it follows from the definition of $\tilde d^m_T$ and the triangle inequality that
\begin{align*}
	\tilde d^m_T(\cX^{n,\kappa},  \cX^{n,\kappa}_{\rho_n}) & \le \sup_{t \in [0, T]} |\lambda(t) - t| + \sup_{t \in [0, T]} |\cX^{n,\kappa}_t - \cX^{n,\kappa}_{\rho_n(\lambda(t))}|\\
	& \le 2 |\Pi_n| + \sup_{t \in [0, T]} |\cX^{n,\kappa}_{C, t} - \cX^{n,\kappa}_{C, \rho_n(\lambda(t))}|  + \sup_{t \in [0, T]} |\cX^{n,\kappa}_{J, t} - \cX^{n,\kappa}_{J, \rho_n(\lambda(t))}| \\
	& = 2 |\Pi_n| + \max_{1 \le i \le n} \; \sup_{t \in (t^n_{i-1}, t^n_i]} |\cX^{n,\kappa}_{C,t} - \cX^{n,\kappa}_{C, \rho_n(\lambda(t))}| + \max_{1 \le i \le n} \; \sup_{t \in [\tau^n_{i-1}, \tau^n_i)} |\cX^{n,\kappa}_{J,t} - \cX^{n,\kappa}_{J, \rho_n(\lambda(t))}|.
\end{align*}
Notice that $t^n_{i-1} \in [\tau^n_{i-1}, \tau^n_i)$, $\lambda(t) \in [t^n_{i-1}, t^n_i)$ for $t \in [\tau^n_{i-1}, \tau^n_i)$, and on the event $\cap_{i = 1}^n A^{n, \kappa}_i$, $\cX^{n,\kappa}_{J}$ is constant on $[\tau^n_{i-1}, \tau^n_i)$ as it does not have jumps on $(\tau^n_{i-1}, \tau^n_i)$, it thus implies that 
\begin{align*}
	\cX^{n, \kappa}_{J, t} = \cX^{n, \kappa}_{J, t^n_{i-1}} = \cX^{n, \kappa}_{J, \rho_n(\lambda(t))}, \quad t \in [\tau^n_{i-1}, \tau^n_i).
\end{align*}
Moreover, for $i = 1, \ldots, n$ and $t \in (t^n_{i-1}, t^n_i]$, we observe that 
\begin{align*}
	& \trm{for } t \in (t^n_{i-1}, \tau^n_i): \quad  t^n_{i-1} < \lambda(t) < t^n_i,\\
	& \trm{for } t \in [\tau^n_i, t^n_i]: \quad t^n_i \le \lambda(t) \le \lambda(t^n_i) \begin{cases}
		< \lambda(\tau^n_{i + 1}) = t^n_{i+1}  & \trm{if } i \le n-1 \\
		= t^n_i &  \trm{if } i = n,
	\end{cases} 
\end{align*}
which implies $\rho_n(\lambda(t)) \in \{t^n_{i-1}, t^n_i\}$ for $t \in (t^n_{i-1}, t^n_i]$. Summarizing those arguments, on $\cap_{i = 1}^n A^{n, \kappa}_i$ we have
\begin{align}\label{eq:Skorokhod-distance-estimate-kappa}
	\tilde d^m_T(\cX^{n,\kappa},  \cX^{n,\kappa}_{\rho_n}) & \le 2 |\Pi_n| + 2 \max_{1 \le i \le n} \, \sup_{t \in (t^n_{i-1}, t^n_i]} |\cX^{n,\kappa}_{C,t} - \cX^{n,\kappa}_{C, t^n_{i-1}}| \notag \\
	& \le 2\bb[|\Pi_n| + \max_{1 \le i \le n} \, \sup_{t \in (t^n_{i-1}, t^n_i]} |\cX^{n,\kappa}_{D,t} - \cX^{n,\kappa}_{D, t^n_{i-1}}| + \max_{1 \le i \le n} \, \sup_{t \in (t^n_{i-1}, t^n_i]} |\cX^{n,\kappa}_{\nu,t} - \cX^{n,\kappa}_{\nu, t^n_{i-1}}| \notag \\
	& \qquad + \max_{1 \le i \le n} \, \sup_{t \in (t^n_{i-1}, t^n_i]} |\cX^{n,\kappa}_{B,t} - \cX^{n,\kappa}_{B, t^n_{i-1}}|\bb] \notag \\
	& \le 2 \bb[|\Pi_n| + \|f_0\|_{B_b(\bfU; \bR^m)}|\Pi_n| + \|f_{p+1}\|_{B_b(\bfV; \bR^m)} \max_{1 \le i \le n} \int_{t^n_{i-1}}^{t^n_i} \int_{\kappa < |z| \le R} |z| \nu_s(\od z) \od s \notag \\
	& \qquad + \sum_{l=1}^p \max_{1 \le i \le n} \, \sup_{t \in (t^n_{i-1}, t^n_i]} \bb|\int_{t^n_{i-1}}^t f_l(s, \xi^n_i) \od B^{(l)}_s \bb| \bb].
\end{align}
For any $\ep >0$,
\begin{align}\label{eq:estimate-probability-kappa}
	\bPn(\{\tilde d^m_T(\cX^{n,\kappa},  \cX^{n,\kappa}_{\rho_n}) > \ep \}) \le \bPn \bb(\bigcup_{i=1}^n (A_i^{n, \kappa})^c\bb)  + \bPn \bb(\{\tilde d^m_T(\cX^{n,\kappa},  \cX^{n,\kappa}_{\rho_n}) > \ep\} \cap \bigcap_{i=1}^n A^{n, \kappa}_i\bb).
\end{align}
For the first term on the right-hand side, letting $x_i: = \int_{t^n_{i-1}}^{t^n_i} \int_{|z| > \kappa} \nu_s(\od z) \od s$ and using the inequality $\e^x - 1 - x \le \frac{1}{2} \e^{K} x^2$ for $ x\in [0, K]$, we obtain
\begin{align*}
	\bPn \bb(\bigcup_{i=1}^n (A_i^{n, \kappa})^c\bb) & \le \sum_{i=1}^n (1 - \bPn(A^{n, \kappa}_i))   = \sum_{i = 1}^{n-1} (1 - \e^{-x_i} - x_i \e^{-x_i})  + 1 - \e^{-x_n}  \\
	& \le \frac{1}{2} \e^{\max_{1 \le i \le n-1} x_i} \sum_{i=1}^{n-1}  \e^{-x_i}  x_i^2 + x_n  \le \frac{1}{2} \e^{\max_{1 \le i \le n-1} x_i} \max_{1 \le i \le n-1} x_i \sum_{i=1}^{n -1} x_i + x_n.
\end{align*}
Since $\int_0^T \int_{|z| > \kappa} \nu_s(\od z) \od s < \infty$ which ensures the uniform continuity of $[0, T] \ni t \mapsto \int_0^t \int_{|z| > \kappa} \nu_s(\od z) \od s$, we deduce that $\max_{1 \le i \le n} x_i \to 0$ as $n \to \infty$. Hence, 
\begin{align}\label{eq:estimate-probability-kappa-1}
	\bPn\bb(\bigcup_{i=1}^n (A_i^{n, \kappa})^c\bb) \to 0 \quad \trm{as } n \to \infty.
\end{align}
For the second term, since $\max_{1 \le i \le n} \int_{t^n_{i-1}}^{t^n_i} \int_{\kappa < |z| \le R} |z| \nu_s(\od z) \od s \to 0$  as $n \to \infty$ due to the uniform continuity, we deduce from \eqref{eq:Skorokhod-distance-estimate-kappa} that, when $n$ is sufficiently large, 
\begin{align*}
	\bPn\bb(\{\tilde d^m_T(\cX^{n,\kappa},  \cX^{n,\kappa}_{\rho_n}) > \ep\} \cap \bigcap_{i=1}^n A^{n, \kappa}_i\bb) & \le \bPn \bb(\bb\{\sum_{l=1}^p \max_{1 \le i \le n} \, \sup_{t \in (t^n_{i-1}, t^n_i]} \bb|\int_{t^n_{i-1}}^t f_l(s, \xi^n_i) \od B^{(l)}_s \bb| > \frac{\ep}{4} \bb\}\bb).
\end{align*}
Applying the Burkholder--Davis--Gundy inequality with the exponent $4$ yields
\begin{align}
	&\bPn\bb(\{\tilde d^m_T(\cX^{n,\kappa},  \cX^{n,\kappa}_{\rho_n}) > \ep\} \cap \bigcap_{i=1}^n A^{n, \kappa}_i \bb) \notag \\
	& \le \sum_{l=1}^p \sum_{i=1}^n 	\bPn \bb(\bb\{\sup_{t \in (t^n_{i-1}, t^n_i]} \bb|\int_{t^n_{i-1}}^t f_l(s, \xi^n_i) \od B^{(l)}_s \bb| > \frac{\ep}{4 p} \bb\}\bb) \notag\\
	& \le c\frac{256 p^4}{\ep^4} \sum_{l=1}^p \sum_{i=1}^n \bEn\bb[\bb| \int_{t^n_{i-1}}^{t^n_i}|f_l(s, \xi^n_i)|^2 \od s \bb|^2\bb] \notag\\
	& \le c \frac{256 p^5}{\ep^4} \max_{1 \le l \le p} \|f_l\|_{B_b(\bfU;\bR^m)}^4 \sum_{i=1}^n (t^n_i - t^n_{i-1})^2 \notag\\
	& \le c \frac{256 p^5 T}{\ep^4} \max_{1 \le l \le p} \|f_l\|_{B_b(\bfU;\bR^m)}^4 |\Pi_n| \xrightarrow{n \to \infty} 0, \label{eq:estimate-probability-kappa-2}
\end{align}
where $c>0$ is a constant independent of $\ep, n, p, T$. Combining \eqref{eq:estimate-probability-kappa-1} and \eqref{eq:estimate-probability-kappa-2} with \eqref{eq:estimate-probability-kappa}, and then plugging them together with \eqref{eq:estimate-Skorokhod-Y-2} into  \eqref{eq:estimate-Skorokhod-Y-1} we arrive at
\begin{align*}
	\limsup_{n \to \infty} \bPn(\{\tilde d^m_T(\cX^n,  \cX^n_{\rho_n}) > 3 \ep\}) \le 4 \ep^{-2} \|f_{p+1}\|_{B_b(\bfV; \bR^m)}^2 \int_0^T  \nq \int_{0 < |z| \le\kappa} |z|^2 \nu_s(\od z) \od s.
\end{align*}
Letting $\kappa \downarrow 0$ and exploiting \eqref{assumption:Levy-measure} we eventually obtain
\begin{align*}
	\limsup_{n \to \infty} \bPn(\{\tilde d^m_T(\cX^n,  \cX^n_{\rho_n}) > 3 \ep\}) = 0,
\end{align*}
which then verifies \eqref{converge-prob-error}. \qed

\subsection{Proof of assertion (\ref{eq:weak-convergence-rho})}\label{sec:dist_proof}

For the proof of \eqref{eq:weak-convergence-rho}, we apply a limit theorem of Jacod and Shiryaev, which is briefly reviewed in \cref{sec:JS03-limit-theorem}. It relies on verifying the convergence of the modified semimartingale characteristics of $\cX^n_{\rho_n}$ to the modified semimartingale characteristics of the limit process  $\cX$. Here, ``modified'' is understood in the sense of \cite[Definition II.2.16]{JS03}.

Let us fix a truncation function $\frh \colon \bR^m \to \bR^m$, see \cite[Definition II.2.3]{JS03}, i.e. $\frh$ is bounded and $\frh(z) = z$ in a neighborhood of $0$. It is convenient for us to assume furthermore that $\frh^{(k)} \in C^2_b(\bR^m)$ for any $k = 1, \ldots, m$.

The following lemma states the semimartingale characteristics of $\cX$ with respect to the truncation function $\frh$, compare \cite[Definition II.2.6]{JS03}. Its proof follows routine arguments and can be found in \cref{sec:proof-semimartingale-characteristics}.

\begin{lemm}\label{lemm:characteristic-limit}
	$\cX$ is an $m$-dimensional semimartingale whose characteristics $(\frb^{\cX}, C^{\cX}, \nu^{\cX})$ with respect to the truncation function $\frh$ is given by
	\begin{align*}
		&\frb^{\cX}_t   = \int_0^t \bb[\int_{[0, 1]^d} f_0(s, u) \od u  + \int_{\{|z| > R\} \times [0, 1]^d} \frh(f_{p+2}(s, z, u))  \nu_s(\od z) \od u \\
		& \hspace{50pt} + \int_{\{0 < |z| \le R\} \times [0, 1]^d} [\frh(f_{p+1}(s, z, u)|z|)  - f_{p+1}(s, z, u)|z|] \nu_s(\od z) \od u \bb] \od s,\\
		&C^{\cX}_t   =  \bb(\sum_{l=1}^p \int_0^t \nq \int_{[0, 1]^d}  (f^{(k)}_l f^{(k')}_l)(s, u) \od u \od s\bb)_{k, k'} \in \bR^{m \times m}, \quad 0 \le t \le T,\\
		&\nu^{\cX}((s, t] \times A)   = \int_s^t \nq \int_{\{0 < |z| \le R\} \times [0, 1]^d} \1_{A}(f_{p+1}(r, z, u)|z|) \nu_r(\od z) \od u \od r\\
		&  \hspace{85pt} + \int_s^t \nq \int_{\{|z| > R\} \times [0, 1]^d} \1_{A}(f_{p+2}(r, z, u)) \nu_r(\od z) \od u \od r
	\end{align*}
	for $0 \le s < t \le T$, $A \in \cB(\bR^m_0)$.
\end{lemm}

\begin{rema}\label{remark:integration-nu-fM} By a standard approximation argument, the measure $\nu^{\cX}$ in \cref{lemm:characteristic-limit} satisfies
	\begin{align*}
		&\int_0^T \nq \int_{\bR^m_0} g(y) \nu^{\cX}(\od s, \od y)\\
		& = \int_0^T\nq \int_{\bR^q_0 \times [0, 1]^d} [g(f_{p+1}(s, z, u)|z|)\1_{\{0 < |z| \le R\}} + g(f_{p+2}(s, z, u))\1_{\{|z| > R\}}] \nu_s(\od z) \od u \od s
	\end{align*}
	for any measurable $g\colon \bR^m_0 \to \bR$ which is non-negative or $g \1_{[0, T]}$ is $\nu^{\cX}$-integrable. In particular, for $g(y) = \1_{\{|y| \ge \kappa\}}$ with some $\kappa >0$ we get
	\begin{align}\label{rema:levy-measure-characteristic}
		&\int_0^T \nq \int_{|y| \ge \kappa} \nu^{\cX}(\od s, \od y) \notag\\
		&  = \int_0^T\nq \int_{\bR^q_0 \times [0, 1]^d} [\1_{\{|f_{p+1}(s, z, u)| |z| \ge \kappa\}} \1_{\{0 < |z| \le R\}} + \1_{\{|f_{p+2}(s, z, u)| \ge \kappa\}}\1_{\{|z| >R\}}] \nu_s(\od z) \od u \od s \notag\\
		& \le \frac{\|f_{p+1}\|_{B_b(\bfV; \bR^m)}^2}{\kappa^2} \int_0^T\nq \int_{0 < |z| \le R} |z|^2 \nu_s(\od z) \od s  + \frac{\|f_{p+2}\|_{B_b(\bfV; \bR^m)}}{\kappa} \int_0^T\nq \int_{|z| > R} \nu_s(\od z)  \od s \\
		& < \infty, \notag
	\end{align}
	where the finiteness can be derived from \eqref{assumption:Levy-measure} and the inequalities
	$$\1_{\{|f_{p+1}(s, z, u)||z| \ge \kappa\}} \le \kappa^{-2}\|f_{p+1}\|_{B_b(\bfV; \bR^m)}^2|z|^2 \quad \trm{and} \quad \1_{\{|f_{p+2}(s, z, u)| \ge \kappa\}} \le \kappa^{-1}\|f_{p+2}\|_{B_b(\bfV; \bR^m)}.$$
\end{rema}

\medskip

We now turn to $\cX^n_{\rho_n}$, whose modified semimartingale characteristics will be computed in relation to a new filtration, which we construct next. To this end, 
 we set
\begin{align*}
	\sigma_n(t): = \sup\{i : t^n_i \le t\} \in \{0, 1, \ldots, n\},\quad t \in [0, \infty).
\end{align*}
Denote $\Delta^n_i \cX^n : = \cX^n_{t^n_i} - \cX^n_{t^n_{i-1}}$. Then
\begin{align*}
	\cX^n_{\rho_n(t)} = \sum_{i=1}^{\sigma_n(t)} \Delta^n_i \cX^n, \quad t \in [0, T].
\end{align*}
For $n \ge 1$, we define the discrete-time filtration $(\cG^n_i)_{i=0}^n$ by
\begin{align*}
	\cG^n_0 & : = \{\emptyset, \Omega\}, \quad \cG^n_i := \sigma\{\Delta^n_j \cX^n, j \le i\}, \quad i = 1, \ldots, n.
\end{align*}
Then $\{\Delta^n_i \cX^n, \cG^n_i : 1 \le i \le n, n \ge 1\}$ is an adapted triangular array. 
Since $\Delta^n_i \cX^n$ is independent of $\cG^n_{i-1}$, we get for any bounded measurable $g$ and $t \in [0, \infty)$ that, a.s.,
\begin{align*}
	\sum_{i=1}^{\sigma_n(t)} \bEn [g(\Delta^n_i \cX^n) |\cG^n_{i-1}]  = \sum_{i=1}^{\sigma_n(t)} \bEn [g(\Delta^n_i\cX^n)].
\end{align*}

\begin{rema}\label{rema:semimartingale-array}
	\begin{enumerate}
		\item By \cite[Ch.II, §3b]{JS03}, the modified semimartingale characteristics of  $\cX^n_{\rho_n}$ with respect to the filtration $\bG^{\sigma_n}=(\cG^n_{\sigma_n(t)})_{t\geq 0}$ is the triplet (drift part, modified diffusion part, jump part) which is respectively described by
			\begin{align*}
				&\sum_{i=1}^{\sigma_n} \bEn[\frh(\Delta^n_i \cX^n)], \\ & \bb(\sum_{i=1}^{\sigma_n} \b( \bEn[(\frh^{(k)} \frh^{(k')})(\Delta^n_i \cX^n)]- \bEn[\frh^{(k)} (\Delta^n_i \cX^n)] \bEn[ \frh^{(k')}(\Delta^n_i \cX^n)] \b)\bb)_{k,k'=1,\ldots,m}, \\
				&  \sum_{i=1}^{\sigma_n}  \bEn[g(\Delta^n_i \cX^n)],
			\end{align*}
			where $g$ runs through a sufficiently large class of test functions vanishing around zero. 
			
			\item A key difference between $\bG^{\sigma_n}$ and $\bF^{\Pi_n}$ is that information about the random variable $\xi^{\Pi_n}_{t_{i}^n}$, which is sampled for the randomization on the interval $(t^n_{i-1}, t^n_i]$, is only revealed at time $t^n_i$ in the filtration 
			$\bG^{\sigma_n}$, whereas it is already known at time $t^n_{i-1}$ in the filtration $\bF^{\Pi_n}$.
	\end{enumerate}
\end{rema}

The following proposition plays the key role for deriving the convergence of the semimartingale characteristics.
\begin{prop}\label{prop:limit-test-function-g}
	For any $g \in C^2_b(\bR^{m})$, one has
	\begin{align}\label{eq:prop:discrete-time-convergence}
		\sum_{i=1}^{n} \bb| \bEn [g(\Delta^n_i \cX^n)] - g(0) -   \int_{t^n_{i-1}}^{t^n_i} \Psi_{f}(g) (s)  \od s \bb| \xrightarrow{n \to \infty} 0,
	\end{align}
	where the function $\Psi_f(g) \colon [0, T] \to \bR$ is defined by
	\begin{align}\label{eq:defi-Phi}
		\Psi_f(g) (s) & : =   \int_{[0, 1]^d} \bb(\nabla g(0)^\tran f_0(s, u) + \frac{1}{2} \sum_{k, k' = 1}^m \pd^2_{k,k'} g(0) \sum_{l=1}^p  (f^{(k)}_l f^{(k')}_l)(s, u)\bb) \od u \notag \\ 
		& \quad  + \int_{\{0 < |z| \le R\} \times [0, 1]^d } \b[g(f_{p+1}(s, z, u) |z|) - g(0) - |z|  \nabla g(0)^\tran f_{p+1}(s, z, u)  \b] \nu_s(\od z) \od u \notag \\
		& \quad + \int_{\{|z| > R\} \times [0, 1]^d} \b[g(f_{p+2}(s, z, u)) - g(0) \b] \nu_s(\od z) \od u.
	\end{align}
	Consequently, for any $t \in [0, \infty)$,
	\begin{align*}
		\sum_{i=1}^{\sigma_n(t)} \bEn [g(\Delta^n_i \cX^n)] \xrightarrow{n \to \infty}  g(0) +  \int_0^{t \wedge T} \Psi_f(g) (s)  \od s.
	\end{align*}
\end{prop}

\begin{proof} 
	
	\textit{\textbf{Step 1.}} It is obvious that $\Psi_f(g)$ is measurable by Fubini's theorem, and moreover, there exists a constant $c_{T, m}>0$ such that
	\begin{align*}
		\int_0^T |\Psi_f(g) (s)| \od s & \le c_{T, m} \bb(\|f_0\|_{B_b(\bfU; \bR^m)}  |\nabla g(0)| +  \sum_{k, k'=1}^m |\pd^2_{k,k'} g(0)| \sum_{l=1}^p \|f^{(k)}_l f^{(k')}_l\|_{B_b(\bfU)}  \\
		& \quad + \|f_{p+1}\|_{B_b(\bfV;\bR^m)}^2 \|\nabla^2 g\|_{B_b(\bR^m; \bR^{m \times m})}  \int_0^T \nq \int_{\{0 < |z| \le R\}\times [0, 1]^d } |z|^2 \nu_s(\od z) \od u \od s\\
		&  \quad  + 2\|g\|_{B_b(\bR^m)} \int_0^T \nq \int_{\{|z| > R\}\times [0, 1]^d } \nu_s(\od z) \od u \od s\bb)\\
		& < \infty,
	\end{align*}
	Next, for $n \ge 1, i = 1, \ldots, n$, we define the c\`adl\`ag and $\bF^{\Pi_n}$-adapted process $F^{n, i} = (F^{n, i}_t)_{t \in [t^n_{i-1}, t^n_i]}$ null at $t^n_{i-1}$ by setting, for $t \in (t^n_{i-1}, t^n_i]$,
	\begin{align*}
		F^{n, i}_t  & :=   \int_{t^n_{i-1}}^{t} f_0(s, \xi^n_{i}) \od s  + \sum_{l=1}^p  \int_{t^n_{i-1}}^{t} f_l (s, \xi^n_{i}) \od B^{(l)}_s \\
		& \quad +  \int_{t^n_{i-1}}^{t} \int_{0 < |z| \le R} f_{p+1}(s, z, \xi^n_i) |z| \tilde N(\od s, \od z) +  \int_{t^n_{i-1}}^{t} \int_{|z| > R} f_{p+2}(s, z, \xi^n_i) N(\od s, \od z).
	\end{align*}
	Let $s \in (0, T]$ be now fixed. Then,  for any $n \ge 1$, there exists uniquely $1 \le i(s, n) \le n$ such that 
	\begin{align*}
		s \in (t^n_{i(s, n)-1}, t^n_{i(s, n)}] \quad \trm{and} \quad \lim_{n \to \infty} t^n_{i(s, n)-1} = \lim_{n \to \infty} t^n_{i(s, n)} = s.
	\end{align*}
	We claim that
	\begin{align*}
		F^{n, i(s, n)}_s \xrightarrow{\bfL^1(\bPn)} 0 \quad \trm{as } n \to \infty. 
	\end{align*}
	It is straightforward to check when $n \to \infty$ that, in the representation of  $F^{n, i(s, n)}_s$, the Lebesgue integral part tends to $0$ in $\bfL^2(\bPn)$ as $f_0$ is bounded, the martingale part converges to $0$ in $\bfL^2(\bPn)$ by applying It\^o's isometry and using the boundedness of $f_l$, $l = 1, \ldots, p+1$. For the ``large jump part'', since $\nu_r(\od z) \od r$ is the predictable compensator of $N(\od r, \od z)$, together with \eqref{assumption:Levy-measure}, we get
	\begin{align*}
		& \bEn\bb[\bb| \int_{t^n_{i(s, n)- 1}}^s \int_{|z| > R} f_{p+2}(r, z, \xi^n_{i(s, n)}) N(\od r, \od z) \bb|\bb]\\
		& \le \|f_{p+2}\|_{B_b(\bfV; \bR^m)} \bEn\bb[\int_{t^n_{i(s, n)- 1}}^s \int_{|z| > R} N(\od r, \od z) \bb] \\
		& = \|f_{p+2}\|_{B_b(\bfV; \bR^m)} \int_{t^n_{i(s, n)- 1}}^s \int_{|z| > R} \nu_r(\od z) \od r \xrightarrow{n \to \infty} 0,
	\end{align*}
	which verifies the claim. Since $\bEn[N(\{s\} \times \bR^q_0)] = \nu(\{s\} \times \bR^q_0) = 0$, it holds that $F^{n, i(s, n)}_{s} = F^{n, i(s, n)}_{s-}$ a.s.,
	and hence,
	\begin{align}\label{eq:convergence-in-probability-F-2}
		F^{n, i(s, n)}_{s-} \xrightarrow{\bfL^1(\bPn)} 0 \quad \trm{as } n \to \infty. 
	\end{align}
	
	\textbf{\textit{Step 2.}} Using It\^o's formula for $F^{n, i}$ and $g \in C^2_b(\bR^m)$ (see, e.g., \cite[Theorem 2.5]{Ku04}) we get, a.s.,
	\begin{align*}
		& g(\Delta^n_i \cX^n)  =  g\b(F^{n, i}_{t^n_i}\b)\\
		& = g(0) +   \int_{t^n_{i-1}}^{t^n_i} \nabla  g(F^{n, i}_{s-})^\tran f_0(s, \xi^n_i) \od s\\
		& \quad +  \sum_{l = 1}^p \int_{t^n_{i-1}}^{t^n_i} \nabla g(F^{n, i}_{s-})^\tran f_l (s, \xi^n_i) \od B^{(l)}_s + \frac{1}{2} \sum_{k, k' =1}^m \sum_{l=1}^p \int_{t^n_{i-1}}^{t^n_i} \pd^2_{k,k'} g(F^{n, i}_{s-}) (f^{(k)}_l f^{(k')}_l) (s, \xi^n_i) \od s\\
		& \quad + \int_{t^n_{i-1}}^{t^n_i} \int_{0 < |z| \le R} \B[ g\b(F^{n, i}_{s-} + f_{p+1}(s, z, \xi^n_i) |z|\b) - g(F^{n, i}_{s-}) \B] \tilde N(\od s, \od z)\\
		& \quad + \int_{t^n_{i-1}}^{t^n_i} \int_{0 < |z| \le R} \B[ g\b(F^{n, i}_{s-} + f_{p+1}(s, z, \xi^n_i) |z|\b) - g(F^{n, i}_{s-}) - |z| \nabla g(F^{n, i}_{s-})^\tran f_{p+1}(s, z, \xi^n_i)\B] \nu_s(\od z) \od s\\
		& \quad + \int_{t^n_{i-1}}^{t^n_i} \int_{|z| > R} \B[ g\b(F^{n, i}_{s-} + f_{p+2}(s, z,  \xi^n_i) \b) - g(F^{n, i}_{s-}) \B] N(\od s, \od z).
	\end{align*}
	Since $\nabla g$  and $f_l$ are bounded for any $l = 1, \ldots, p+1$, the integrals with respect to the Brownian motions and the compensated random measure are square integrable martingales which vanish after taking the expectation $\bEn$. Let us now investigate the remaining parts.\\
	$\bullet$ \textit{The ``drift part'':} Using Fubini's theorem and the Cauchy--Schwarz inequality yields
	\begin{align*}
		&\sum_{i=1}^{n} \bb| \bEn\bb[\int_{t^n_{i-1}}^{t^n_i} \nabla g(F^{n, i}_{s-})^\tran f_0(s, \xi^n_i) \od s \bb]  - \int_{t^n_{i-1}}^{t^n_i} \int_{[0, 1]^d} \nabla g(0)^\tran f_0(s, u) \od u\od s \bb|\\
		&= \sum_{i=1}^{n} \bb| \bEn\bb[\int_{t^n_{i-1}}^{t^n_i} \nabla g(F^{n, i}_{s-})^\tran f_0(s, \xi^n_i) \od s  - \int_{t^n_{i-1}}^{t^n_i} \nabla g(0)^\tran f_0(s, \xi^n_i) \od s  \bb] \bb|\\
		& \le  \|f_0\|_{B_b(\bfU; \bR^m)} \sum_{i=1}^{n} \int_{t^n_{i-1}}^{t^n_i} \bEn[|\nabla g(F^{n, i}_{s-}) - \nabla g(0)|] \od s \\
		& = \|f_0\|_{B_b(\bfU; \bR^m)} \int_0^{T} \bEn \bb[\sum_{i=1}^n |\nabla g(F^{n, i}_{s-}) - \nabla g(0)| \1_{(t^n_{i-1}, t^n_i]}(s)\bb] \od s \\
		& = \|f_0\|_{B_b(\bfU; \bR^m)} \int_0^{T} \bEn \b[\b|\nabla g \b(F^{n, i(s, n)}_{s-}\b) - \nabla g(0)\b|\b] \od s\\
		& \xrightarrow{n \to \infty} 0,
	\end{align*}
	where we apply the dominated convergence theorem using \eqref{eq:convergence-in-probability-F-2} together with the continuity and boundedness of $\nabla g$. Analogously, for $k, k' = 1, \ldots, m$ and $l = 1, \ldots, p$,
	\begin{align*}
		\sum_{i=1}^{n} \bb| \bEn\bb[\int_{t^n_{i-1}}^{t^n_i} \pd^2_{k,k'} g(F^{n, i}_{s-}) (f^{(k)}_l f^{(k')}_l)(s, \xi^n_i) \od s\bb]  - \int_{t^n_{i-1}}^{t^n_i} \int_{[0, 1]^d}  \pd^2_{k,k'} g(0) (f^{(k)}_l f^{(k')}_l)(s, u) \od u\od s   \bb| \xrightarrow{n 
			\to \infty} 0.
	\end{align*}
	$\bullet$ \textit{The ``small jump part'':} For $i(n, s)$ introduced in \textbf{\textit{Step 1}} one has
	\begin{align*}
	&\sum_{i=1}^{n} \bb| \bEn\bb[\int_{t^n_{i-1}}^{t^n_i} \int_{0 < |z| \le R} \B[ g\b(F^{n, i}_{s-} + f_{p+1}(s, z, \xi^n_i)|z|\b) - g(F^{n, i}_{s-}) - |z| \nabla g(F^{n, i}_{s-})^\tran f_{p+1}(s, z, \xi^n_i) \B] \nu_s(\od z)\od s\bb] \\
		& \qquad - \int_{t^n_{i-1}}^{t^n_i}\int_{\{0 < |z| \le R\} \times [0, 1]^d}   \B[g(f_{p+1}(s, z,u)|z|) - g(0) - |z| \nabla g(0)^\tran f_{p+1}(s, z, u)\B]  \nu_s(\od z) \od u \od s \bb| \\ 
		& \le \int_0^T \nq \int_{0 < |z| \le R} \bEn\bb[\sum_{i=1}^n \B| g\b(F^{n, i}_{s-} + f_{p+1}(s, z, \xi^n_i)|z|\b) - g(F^{n, i}_{s-}) - |z| \nabla g(F^{n, i}_{s-})^\tran f_{p+1}(s, z,\xi^n_i) \\
		& \qquad - g(f_{p+1}(s, z, \xi^n_i)|z|) + g(0) + |z| \nabla g(0)^\tran f_{p+1}(s, z,\xi^n_i) \B| \1_{(t^n_{i-1}, t^n_i]}(s) \bb] \nu_s(\od z) \od s\\
		& = \int_0^T \nq \int_{0 < |z| \le R} \bEn\B[\B| g\b(F^{n, i(s, n)}_{s-} + f_{p+1}(s, z, \xi^n_{i(s, n)})|z|\b) - g\b(F^{n, i(s, n)}_{s-}\b) - |z| \nabla g\b(F^{n, i(s, n)}_{s-}\b)^\tran f_{p+1}(s, z,\xi^n_{i(s, n)})  \\
		& \qquad - g(f_{p+1}(s, z, \xi^n_{i(s, n)})|z|) + g(0) + |z|\nabla g(0)^\tran f_{p+1}(s, z, \xi^n_{i(s, n)})\B|\B] \nu_s(\od z) \od s\\
		& =: \int_0^T \nq \int_{0 < |z| \le R} \bEn[G^{\trm{S}}_n(s, z)] \nu_s(\od z) \od s.
	\end{align*}
	Using Taylor's expansion we obtain a constant $c_m >0$ depending only on $m$ such that
	\begin{align*}
		G^{\trm{S}}_n(s, z) \le c_m \|\nabla^2 g\|_{B_b(\bR^m; \bR^{m \times m})} \|f_{p+1}\|_{B_b(\bfV; \bR^m)}^2 |z|^2.
	\end{align*}
	Hence, it is easy to check using \eqref{eq:convergence-in-probability-F-2} and dominated convergence that  $\bEn[G^{\trm{S}}_n(s, z)] \to 0$ as $n \to\infty$ for any $s, z$. Due to \eqref{assumption:Levy-measure}, dominated convergence also yields
	\begin{align*}
		\int_0^T \nq \int_{0 < |z| \le R} \bEn[G^{\trm{S}}_n(s, z)] \nu_s(\od z) \od s \xrightarrow{n \to \infty} 0.
	\end{align*}
	$\bullet$ \textit{The ``large jump part'':}
	Since $\nu_s(\od z) \od s$ is the predictable compensator of $N(\od s, \od z)$, using Fubini's theorem, again, for interchanging integrals we get
	\begin{align*}
		& \sum_{i=1}^{n} \bb| \bEn\bb[\int_{t^n_{i-1}}^{t^n_i} \int_{|z| > R} \B[ g\b(F^{n, i}_{s-} + f_{p+2}(s, z,  \xi^n_i)\b) - g(F^{n, i}_{s-}) \B] N(\od s, \od z)\bb]\\
		& \qquad - \int_{t^n_{i-1}}^{t^n_i} \int_{\{|z| > R\} \times [0, 1]^d}  \b[ g(f_{p+2}(s, z,  u)) - g(0) \b] \nu_s(\od z) \od u \od s\bb] \bb|\\
		& = \sum_{i=1}^{n} \bb| \bEn\bb[\int_{t^n_{i-1}}^{t^n_i} \int_{|z| > R} \B[ g\b(F^{n, i}_{s-} + f_{p+2}(s, z, \xi^n_i) \b) - g(F^{n, i}_{s-}) \B] \nu_s(\od z) \od s\bb]\\
		& \qquad - \bEn\bb[\int_{t^n_{i-1}}^{t^n_i} \int_{|z| > R} \b[ g(f_{p+2}(s, z,  \xi^n_i) ) - g(0) \b] \nu_s(\od z) \od s\bb] \bb|\\
		& \le \int_0^T \nq \int_{|z| >R} \bEn\B[\B|  g\b(F^{n, i(s, n)}_{s-} + f_{p+2}(s, z, \xi^n_{i(s, n)})\b) - g\b(F^{n, i(s, n)}_{s-}\b) -   g(f_{p+2}(s, z, \xi^n_{i(s, n)})) + g(0)  \B| \B] \nu_s(\od z) \od s\\
		& =: \int_0^T \nq \int_{|z| >R} \bEn[G^{\trm{L}}_n(s, z)] \nu_s(\od z) \od s.
	\end{align*}
	It is obvious that $G^{\trm{L}}_n$ is uniformly bounded by $4\|g\|_{B_b(\bR^m)}$. Moreover, using the Lipschitzian of $g$, the boundedness of $f_{p+2}$ and \eqref{eq:convergence-in-probability-F-2} and \eqref{assumption:Levy-measure}, we may apply the dominated convergence theorem to obtain   
	\begin{align*}
		\int_0^T \nq \int_{|z| >R} \bEn[G^{\trm{L}}_n(s, z)] \nu_s(\od z) \od s \xrightarrow{n \to \infty} 0.
	\end{align*}
	Combining the arguments above yields \eqref{eq:prop:discrete-time-convergence}. The consequence follows from $\int_{\rho_n(t)}^t |\Psi_f(g)(s)| \od s \to 0$ as $n \to \infty$.
\end{proof}

We apply \cref{prop:limit-test-function-g} in the next three lemmas, to prove convergence of the drift part, the modified diffusion part, and the jump part of the semimartingale characteristics as aforementioned in \cref{rema:semimartingale-array}(1).

\begin{lemm}\label{lem:convergence-drift}
	For $\frb^{\cX}$ in \cref{lemm:characteristic-limit} and any $t \in [0, \infty)$, 
	\begin{align}\label{eq:limit-drift}
		I_{\eqref{eq:limit-drift}}: = \sup_{0 \le s \le t} \bb|\sum_{i=1}^{\sigma_n(s)} \bEn[\frh(\Delta^n_i \cX^n)] - \frb^{\cX}_{s \wedge T}\bb| \xrightarrow{n \to \infty} 0.
	\end{align}
\end{lemm}

\begin{proof}
	It is sufficient to verify the convergence for any $k$-th coordinate, $k = 1, \ldots, m$ and $t \in [0, T]$. Observe that
	\begin{align*}
		\frb^{\cX, (k)}_t = \int_0^t \Psi_{f}(\frh^{(k)})(s) \od s
	\end{align*}
	for $\Psi_f(\frh^{(k)})$ associated with $\frh^{(k)}$ introduced in \cref{prop:limit-test-function-g}.
	Then we get
	\begin{align*}
		& \sup_{0 \le s \le t} \bb|\sum_{i=1}^{\sigma_n(s)} \bEn[\frh^{(k)}(\Delta^n_i \cX^n)] - \frb^{\cX, (k)}_s\bb| \\
		& \le \sup_{0 \le s \le t} \bb|\sum_{i=1}^{\sigma_n(s)} \bEn[\frh^{(k)}(\Delta^n_i \cX^n)] - \sum_{i=1}^{\sigma_n(s)} \int_{t^n_{i-1}}^{t^n_i} \Psi_f(\frh^{(k)})(r) \od r\bb| + \sup_{0 \le s \le t} \bb| \int_{\rho_n(s)}^s \Psi_f(\frh^{(k)})(r) \od r \bb|\\
		& \le  \sum_{i=1}^{n} \bb| \bEn[\frh^{(k)}(\Delta^n_i \cX^n)] -  \int_{t^n_{i-1}}^{t^n_i} \Psi_f(\frh^{(k)})(r) \od r\bb| + \max_{1 \le i \le n}  \int_{t^n_{i-1}}^{t^n_i} |\Psi_f(\frh^{(k)})(r)| \od r.
	\end{align*}
	The first term on the right-hand side above converges to $0$ by applying \cref{prop:limit-test-function-g} for $\frh^{(k)} \in C^2_b(\bR^m)$. For the second term, since $t \mapsto \int_0^t |\Psi_f(\frh^{(k)})(r)| \od r$ is uniformly continuous on $[0, T]$ and $\max_{1 \le i \le n}|t^n_i - t^n_{i-1}| \to 0$, it implies that
	\begin{align*}
		\max_{1 \le i \le n} \int_{t^n_{i-1}}^{t^n_i} |\Psi_f(\frh^{(k)})(r)| \od r \xrightarrow{n \to \infty} 0.
	\end{align*} 
	Therefore, $I_{\eqref{eq:limit-drift}} \to 0$ as $n \to \infty$.
\end{proof}

\begin{lemm}\label{lem:convergence-diffusion}
	For $C^{\cX}$ given in \cref{lemm:characteristic-limit}, for any $t \in [0, \infty)$ and $k, k' = 1, \ldots, m$, one has
	\begin{align}
		I_{\eqref{eq:limit-diffusion-test-1}}&:= \sum_{i=1}^{\sigma_n(t)} \bEn[\frh^{(k)}(\Delta^n_i \cX^n)] \bEn[\frh^{(k')}(\Delta^n_i \cX^n)] \xrightarrow{n \to \infty} 0, \label{eq:limit-diffusion-test-1}\\
		I_{\eqref{eq:limit-diffusion-test-2}} &:= \sum_{i=1}^{\sigma_n(t)} \bEn[(\frh^{(k)} \frh^{(k')})(\Delta^n_i \cX^n)] \xrightarrow{n \to \infty}  C^{\cX, (k, k')}_{t \wedge T} + \int_0^{t \wedge T} \nq \int_{\bR^m_0} (\frh^{(k)} \frh^{(k')})(y) \nu^{\cX} (\od s, \od y). \label{eq:limit-diffusion-test-2}
	\end{align}
\end{lemm}

\begin{proof} It suffices to show the convergences for $t \in [0, T]$.
	For $I_{\eqref{eq:limit-diffusion-test-1}}$, we first express
	\begin{align*}
		I_{\eqref{eq:limit-diffusion-test-1}} & =  \sum_{i=1}^{\sigma_n(t)} \bb( \bEn[\frh^{(k)}(\Delta^n_i \cX^n)] - \int_{t^n_{i-1}}^{t^n_i} \Psi_f(\frh^{(k)})(s) \od s \bb) \bEn[\frh^{(k')}(\Delta^n_i \cX^n)] \\
		& \quad +  \sum_{i=1}^{\sigma_n(t)} \bb( \bEn[\frh^{(k')}(\Delta^n_i \cX^n)] - \int_{t^n_{i-1}}^{t^n_i} \Psi_f(\frh^{(k')})(s) \od s \bb) \int_{t^n_{i-1}}^{t^n_i} \Psi_f(\frh^{(k)})(s) \od s\\
		&  \quad + \sum_{i=1}^{\sigma_n(t)} \bb(\int_{t^n_{i-1}}^{t^n_i} \Psi_f(\frh^{(k)})(s) \od s \bb) \bb( \int_{t^n_{i-1}}^{t^n_i} \Psi_f(\frh^{(k')})(s) \od s\bb).
	\end{align*}
	Hence, the triangle inequality yields
	\begin{align*}
		|I_{\eqref{eq:limit-diffusion-test-1}}| & \le \|\frh^{(k')}\|_{B_b(\bR^m)} \sum_{i=1}^{n} \bb| \bEn[\frh^{(k)}(\Delta^n_i \cX^n)] - \int_{t^n_{i-1}}^{t^n_i} \Psi_f(\frh^{(k)})(s) \od s \bb| \\
		& \quad +  \bb(\int_0^T |\Psi_f(\frh^{(k)})(s)| \od s\bb) \sum_{i=1}^{n} \bb| \bEn[\frh^{(k')}(\Delta^n_i \cX^n)] - \int_{t^n_{i-1}}^{t^n_i} \Psi_f(\frh^{(k')})(s) \od s \bb|\\
		&  \quad +  \bb(\int_0^T |\Psi_f(\frh^{(k)})(s)|  \od s\bb) \max_{1 \le i \le n} \int_{t^n_{i-1}}^{t^n_i} |\Psi_f(\frh^{(k')})(s)| \od s.
	\end{align*}
	Applying \cref{prop:limit-test-function-g} for $\frh^{(k)}, \frh^{(k')} \in C^2_b(\bR^m)$, we obtain that the sums $\sum_{i = 1}^{n}$ in the first two terms on the right-hand side converge to $0$ as $n \to \infty$.
	Since $\max_{1 \le i \le n} \int_{t^n_{i-1}}^{t^n_i} |\Psi_f(\frh^{(k')})(s)| \od s \to 0$, we derive $I_{\eqref{eq:limit-diffusion-test-1}} \to 0$ as desired.
	
	\smallskip
	
	For $I_{\eqref{eq:limit-diffusion-test-2}}$, since $\frh^{(k)} \frh^{(k')} \in C^2_b(\bR^m)$ and $\frh^{(k)} \frh^{(k')}(z) = z^{(k)} z^{(k')}$ around $0$, the function $\Psi_f(\frh^{(k)} \frh^{(k')})$ given in \eqref{eq:defi-Phi} can be explicitly written as
	\begin{align*}
		\Psi_f(\frh^{(k)} \frh^{(k')})(s) & =  \sum_{l=1}^p \int_{[0, 1]^d}  (f^{(k)}_l f^{(k')}_l)(s, u)  \od u \\
		& \quad + \int_{\{0 < |z| \le R\} \times [0, 1]^d} (\frh^{(k)} \frh^{(k')})(f_{p+1}(s, z, u)|z|)  \nu_s(\od z) \od u\\
		& \quad + \int_{\{|z| >R\} \times [0, 1]^d} (\frh^{(k)} \frh^{(k')})(f_{p+2}(s, z, u))  \nu_s(\od z) \od u
	\end{align*}
	so that
	\begin{align*}
		\int_0^t \Psi_f(\frh^{(k)} \frh^{(k')})(s) \od s = C^{\cX, (k, k')}_t + \int_0^t \nq \int_{\bR^m_0} (\frh^{(k)} \frh^{(k')})(y) \nu^{\cX}(\od s, \od y)
	\end{align*}
	where we apply \cref{remark:integration-nu-fM} for the $\nu^{\cX}$-integrable function $\frh^{(k)} \frh^{(k')} \1_{[0, T]}$. Hence, \eqref{eq:limit-diffusion-test-2} follows directly from the consequence in \cref{prop:limit-test-function-g}.
\end{proof}

To investigate the jump part of the limiting process, we recall from \cite[p.395]{JS03} the family $C_2(\bR^m)$ of  bounded and continuous functions $g \colon \bR^m \to \bR$ with $g(0) = 0$ around $0$.

\begin{lemm}\label{lem:convergence-jump}
	For $\nu^{\cX}$ in \cref{lemm:characteristic-limit} and for any $g \in C_2(\bR^m)$, $t \in [0, \infty)$, one has
	\begin{align}\label{eq:convergence-jump}
		I^g_{\eqref{eq:convergence-jump}}: = \bb| \sum_{i=1}^{\sigma_n(t)}  \bEn[g(\Delta^n_i \cX^n)] -  \int_0^{t\wedge T} \nq \int_{\bR^m_0} g(y) \nu^{\cX}(\od s, \od y)\bb| \xrightarrow{n \to \infty} 0.
	\end{align}
\end{lemm}

\begin{proof}  We only need to prove for $t \in [0, T]$.
	\smallskip
	
	\textbf{\textit{Step 1.}}
	Recall $\Delta^n_i \cX^n$ from \eqref{eq:decomposition:discrete-time-integral}. We show that for any $\kappa >0$,
	\begin{align}\label{eq:jump-part-tail-estimate-2}
		\sum_{i=1}^n \bPn(\{|\Delta^n_i \cX^n|  \ge \kappa \})  & \le  \frac{9}{\kappa^2}  \bb(pT \max_{1 \le l \le p} \|f_l\|_{B_b(\bfU; \bR^m)}^2  + \|f_{p+1}\|^2_{B_b(\bfV; \bR^m)}\int_0^T \nq \int_{0 < |z| \le R} |z|^2 \nu_s(\od z) \od s\bb)   \notag \\
		& \quad + \frac{3 T }{\kappa} \|f_0\|_{B_b(\bfU; \bR^m)} + \frac{3}{\kappa} \|f_{p+2}\|_{B_b(\bfV; \bR^m)} \int_0^T \nq \int_{|z|>R} \nu_s(\od z) \od s.
	\end{align}
	Indeed, by the triangle inequality we get
	\begin{align}\label{eq:jump-part-tail-estimate}
		&\sum_{i=1}^n \bPn(\{|\Delta^n_i \cX^n| \ge \kappa\}) \notag \\
		& \le \sum_{i=1}^n \bPn\bb(\bb\{\bb|\int_{t^n_{i-1}}^{t^n_i} f_0(s, \xi^n_i) \od s \bb| \ge \frac{\kappa}{3} \bb\}\bb) \notag \\
		& \quad + \sum_{i=1}^n \bPn\bb(\bb\{\bb|\sum_{l=1}^p \int_{t^n_{i-1}}^{t^n_i} f_l(s, \xi^n_i) \od B^{(l)}_s + \int_{t^n_{i-1}}^{t^n_i} \int_{0 < |z| \le R} f_{p+1}(s, z, \xi^n_i) |z| \tilde N(\od s, \od z)\bb| \ge \frac{\kappa}{3}\bb\}\bb) \notag \\
		& \quad + \sum_{i=1}^n \bPn\bb(\bb\{\bb| \int_{t^n_{i-1}}^{t^n_i} \int_{|z| > R} f_{p+2}(s, z, \xi^n_{i}) N(\od s, \od z)\bb| \ge \frac{\kappa}{3} \bb\}\bb) \notag\\
		& =: I_{\eqref{eq:jump-part-tail-estimate}} + I\!I_{\eqref{eq:jump-part-tail-estimate}}  + I\! I \!I_{\eqref{eq:jump-part-tail-estimate}}.
	\end{align}
	For the first term, Markov's inequality yields
	\begin{align*}
		I_{\eqref{eq:jump-part-tail-estimate}} & \le \frac{3}{\kappa} \sum_{i=1}^n \bEn\bb[\bb|\int_{t^n_{i-1}}^{t^n_i} f_0(s, \xi^n_i)\od s \bb|\bb] \le \frac{3 T }{\kappa} \|f_0\|_{B_b(\bfU; \bR^m)}.
	\end{align*}
	For the second term, applying the Markov's inequality and It\^o's isometry we get
	\begin{align*}
		I\!I_{\eqref{eq:jump-part-tail-estimate}} &  \le \frac{9}{\kappa^2}  \sum_{i=1}^n \bEn\bb[\bb| \sum_{l=1}^p \int_{t^n_{i-1}}^{t^n_i} f_l(s, \xi^n_i) \od B^{(l)}_s + \int_{t^n_{i-1}}^{t^n_i} \int_{0 < |z| \le R} f_{p+1}(s, z, \xi^n_i) |z| \tilde N(\od s, \od z) \bb|^2 \bb]\\
		& = \frac{9}{\kappa^2}  \sum_{i=1}^n \bEn\bb[ \sum_{l=1}^p \int_{t^n_{i-1}}^{t^n_i} |f_l(s, \xi^n_i)|^2 \od s + \int_{t^n_{i-1}}^{t^n_i} \int_{0 < |z| \le R} |f_{p+1}(s, z,\xi^n_i)|^2 |z|^2 \nu_s(\od z) \od s \bb]\\
		& \le \frac{9}{\kappa^2}  \bb(pT \max_{1 \le l \le p} \|f_l\|_{B_b(\bfU; \bR^m)}^2  + \|f_{p+1}\|^2_{B_b(\bfV; \bR^m)}\int_0^T \nq \int_{0 < |z| \le R} |z|^2 \nu_s(\od z) \od s\bb).
	\end{align*}
	For the third term, using Markov's inequality we obtain
	\begin{align*}
		I \! I  \! I_{\eqref{eq:jump-part-tail-estimate}}  & \le  \frac{3}{\kappa} \sum_{i=1}^n \bEn\bb[\bb| \int_{t^n_{i-1}}^{t^n_i} \int_{|z| > R} f_{p+2}(s, z, \xi^n_{i}) N(\od s, \od z)\bb|\bb] \\
		& \le \frac{3}{\kappa} \|f_{p+2}\|_{B_b(\bfV; \bR^m)} \sum_{i=1}^n  \bEn\bb[ \int_{t^n_{i-1}}^{t^n_i} \int_{|z| >R} N(\od s, \od z)\bb]\\
		& =  \frac{3}{\kappa} \|f_{p+2}\|_{B_b(\bfV; \bR^m)} \int_0^T \nq \int_{|z|>R} \nu_s(\od z) \od s.
	\end{align*}
	Hence, combining those four estimates yields \eqref{eq:jump-part-tail-estimate-2}.
	
	\smallskip
	
	\textbf{\textit{Step 2.}}  Since $g \in C_2(\bR^m)$, there is an $r_g >0$ such that $g = 0$ on the open ball  $B_m(r_g)$. Then we use  \cref{remark:integration-nu-fM} to obtain that
	\begin{align*}
		&\int_0^T \nq \int_{|y| \ge r_g} |g(y)| \nu^{\cX}(\od s, \od y)  \le \|g\|_{B_b(\bR^m)} 		\int_0^T \nq \int_{|y| \ge r_g} \nu^{\cX}(\od s, \od y)  < \infty.
	\end{align*}
	Hence, the integral on the right-hand side of \eqref{eq:convergence-jump} finitely exists.
	
	We now only prove \eqref{eq:convergence-jump} in the case $0\le R < \infty$ as the case $R = \infty$ is analogous. Let $\ep >0$ and $\theta > r_g \vee R^2$. Since $g$ is continuous and bounded, there exists a continuous function $g_{\theta}$ with compact support such that 
	\begin{align*}
		\|g_{\theta}\|_{B_b(\bR^m)} \le \|g\|_{B_b(\bR^m)} \quad \trm{and} \quad g_{\theta} = g \; \trm{ on } B_m(\theta).
	\end{align*}
	Moreover, by convolution approximation, we can find a $g_{\ep, \theta} \in C_2(\bR^m) \cap C^2_c(\bR^m)$ such that
	\begin{align*}
		g_{\ep, \theta} = g_{\theta} = 0 \; \trm{ on } B_m(r_g/2), \quad \trm{and}\quad \|g_{\ep, \theta} -  g_{\theta}\|_{B_b(\bR^m)} \le \ep.
	\end{align*}
	It follows from the linearity and the triangle inequality that
	\begin{align}\label{eq:jump-part:estimate}
		I^g_{\eqref{eq:convergence-jump}} & \le I^{g - g_{\theta}}_{\eqref{eq:convergence-jump}} + I^{g_{\theta} - g_{\ep, \theta}}_{\eqref{eq:convergence-jump}} + I^{g_{\ep, \theta}}_{\eqref{eq:convergence-jump}}.
	\end{align}
	Since $g_{\ep, \theta} \in C^2_c(\bR^m)$ takes value $0$ in a neighborhood of $0$, \cref{remark:integration-nu-fM} implies 
	\begin{align*}
		\int_0^t \Psi_f(g_{\ep, \theta}) (s) \od s & = \int_0^t \nq \int_{\bR^q_0 \times [0, 1]^d} [g(f_{p+1}(s, z, u)|z|)\1_{\{0 < |z| \le R\}} + g(f_{p+2}(s, z, u))\1_{\{0 < |z| \le R\}}] \nu_s(\od z) \od u \od s \\
		& = \int_0^t \nq \int_{\bR^m_0} g(y) \nu^{\cX}(\od s, \od y) 
	\end{align*} so that the consequence in \cref{prop:limit-test-function-g} verifies
	\begin{align*}
		I^{g_{\ep, \theta}}_{\eqref{eq:convergence-jump}} \xrightarrow{n \to \infty} 0.
	\end{align*}
	For $I^{g - g_{\theta}}_{\eqref{eq:convergence-jump}}$, one has
	\begin{align*}
		I^{g - g_{\theta}}_{\eqref{eq:convergence-jump}} & \le \sum_{i=1}^n \bEn[|(g - g_{\theta})(\Delta^n_i \cX^n)|]  + \int_0^T \nq \int_{\bR^m_0} |g(y) - g_{\theta}(y)| \nu^{\cX}(\od s, \od y)\\
		& \le \|g - g_{\theta}\|_{B_b(\bR^m)} \bb( \sum_{i=1}^n \bPn(\{|\Delta^n_i \cX^n| \ge \theta\}) + \int_0^T \nq \int_{|y| \ge \theta} \nu^{\cX}(\od s, \od y)\bb).
	\end{align*}
	We let $\kappa = \theta$  in \eqref{eq:jump-part-tail-estimate-2}  to find that $\sum_{i=1}^n \bPn(\{|\Delta^n_i \cX^n| \ge \theta\}) \to 0$ uniformly in $n$ as $\theta \to \infty$. Moreover, it follows from \eqref{rema:levy-measure-characteristic} that $\int_0^T \nq \int_{|y| \ge \theta} \nu^{\cX}(\od s, \od y) \to 0$ as $\theta \to \infty$ which thus yields
	\begin{align*}
		I^{g - g_{\theta}}_{\eqref{eq:convergence-jump}} \to 0 \quad \trm{uniformly in } n \trm{ as } \theta \to \infty.
	\end{align*}
	For $I^{g_{\theta} - g_{\ep, \theta}}_{\eqref{eq:convergence-jump}}$, one has
	\begin{align*}
		I^{g_{\theta} - g_{\ep, \theta}}_{\eqref{eq:convergence-jump}} & \le \sum_{i=1}^n \bEn[|(g_{\theta} - g_{\ep, \theta})(\Delta^n_i \cX^n)|]  + \int_0^T \nq \int_{\bR^m_0} |g_{ \theta}(y) - g_{\ep, \theta}(y)| \nu^{\cX}(\od s, \od y)\\
		& \le \|g_{\theta} - g_{\ep, \theta}\|_{B_b(\bR^m)} \bb( \sum_{i=1}^n \bPn(\{|\Delta^n_i \cX^n| \ge r_g/2\}) + \int_0^T \nq \int_{|y| \ge r_g/2} \nu^{\cX}(\od s, \od y)\bb)\\
		& \le \ep \bb( \sum_{i=1}^n \bPn(\{|\Delta^n_i \cX^n| \ge r_g/2\}) + \int_0^T \nq \int_{|y| \ge r_g/2} \nu^{\cX}(\od s, \od y)\bb).
	\end{align*}
	Choosing $\kappa = r_g/2$ in \eqref{eq:jump-part-tail-estimate-2} and using \eqref{rema:levy-measure-characteristic} we obtain
	\begin{align*}
		I^{g_{\theta} - g_{\ep, \theta}}_{\eqref{eq:convergence-jump}} \to 0 \quad \trm{uniformly over } n \trm{ as } \ep \to 0.
	\end{align*}
	Since $\theta$ can be chosen arbitrarily large and $\ep>0$  arbitrarily small, we derive from \eqref{eq:jump-part:estimate} the desired conclusion. 
\end{proof}

We can now finalize the \emph{proof of assertion (\ref{eq:weak-convergence-rho})}. Combining \cref{lem:convergence-drift,lem:convergence-diffusion,lem:convergence-jump} with \cref{lemm:characteristic-limit}, together with applying \cref{thm:weak-convergence-JS03}, we obtain that $(\sum_{i = 1}^{\sigma_n(t)} \Delta^n_i \cX^n)_{t \in [0, \infty)} \to (\cX_{t \wedge T})_{t \in [0, \infty)}
$ as $n \to \infty$
weakly in the Skorokhod topology on the space $\bD_\infty(\bR^m)$ of c\`adl\`ag functions $F \colon [0, \infty) \to \bR^m$ (see \cite{Bi99,JS03} for $\bD_\infty(\bR^m)$). Since $\cX$ has no fixed time of discontinuity, we use \cite[Theorem 16.7]{Bi99} to infer that $\cX^n_{\rho_n} = 	(\sum_{i = 1}^{\sigma_n(t)} \Delta^n_i \cX^n)_{t \in [0, T]} \xrightarrow{\scrD_T} (\cX_{t})_{t \in [0, T]}$ as $n \to \infty$.\qed

\appendix
\section{Background on martingale measures  and proofs for Subsection \ref{sec:discrete_sample_random_measure}}\label{app:random_measures}

Suppose $T \in (0, \infty)$ and let $\bI = [0, T]$ or $\bI = [0, \infty)$. Let $(\Omega, \cF, \bF,\bP)$  be a filtered probability space satisfying the usual conditions with $\bF = (\cF_t)_{t \in \bI}$. Denote by $\cP_\bF$ the predictable $\sigma$-field on $\Omega \times \bI$ associated with the filtration $\bF$.

\subsection{Background on martingale measures}\label{sec:martingale-measure}
Assume that $(E, d_E)$ is a complete and separable metric space equipped with its Borel $\sigma$-field $\cB(E)$. 

\begin{defi}[\cite{Wa86, KM90}] Assume $M \colon \Omega \times \bI \times \cB(E) \to \bR$.
	
	\begin{enumerate}
		\item $M$ is an \textit{$(\bF, \bP)$-martingale measure} on $\bI \times \cB(E)$ if the following conditions are satisfied: 
		\begin{enumerate}
			\item For $A \in \cB(E)$, $(M(t, A))_{t \in \bI}$ is an $\bfL^2(\bP)$-martingale adapted with $\bF$ and $M(0, A) = 0$;
			
			\item For $t \in \bI$ and disjoint $A, B \in \cB(E)$, one has $M(t, A \cup B) = M(t, A) + M(t, B)$ a.s.;
			
			\item There exists a non-decreasing sequence $(E_n)_{n \in \bN} \subseteq \cB(E)$ such that
			\begin{enumerate}
				\item $\cup_{n \in \bN} E_n = E$;
				
				\item For any $t \in \bI$, $\sup_{A \in \cB(E_n)} \|M(t, A)\|_{\bfL^2(\bP)} < \infty$;
				
				\item For any $t \in \bI$, $n \in \bN$, one has $\|M(t, A_k)\|_{\bfL^2(\bP)} \to 0$ for all decreasing sequence $(A_k)_{k \in \bN} \subseteq \cB(E_n)$ with $\cap_{k \in \bN} A_k = \emptyset$.
			\end{enumerate}
		\end{enumerate}
		\item An $(\bF, \bP)$-martingale measure $M$ is said to be \textit{orthogonal} if $M(\cdot, A) M(\cdot, B)$ is an $(\bF, \bP)$-martingale whenever $A, B \in \cB(E)$ with $A \cap B = \emptyset$.
		
		\item An $(\bF, \bP)$-martingale measure $M$ is said to be \textit{continuous} if $\bI\ni t \mapsto M(t, A)$ is continuous for all $A \in \cB(E)$.
	\end{enumerate} 
\end{defi}

It is obvious that, for a given $(\bF, \bP)$, a martingale measure on $[0, \infty) \times \cB(E)$ is also a martingale measure on $[0, T] \times \cB(E)$ by the restriction on $[0, T]$. Conversely, if $M$ is a martingale measure on $[0, T] \times \cB(E)$, then $\hat M(t, \cdot) : = M(t \wedge T, \cdot)$ is an $(\hat \bF, \bP)$-martingale measure on $[0, \infty) \times \cB(E)$, where $\hat \bF = (\cF_{t \wedge T})_{t \ge 0}$.

It is indicated by Walsh \cite{Wa86} (see also \cite[Theorem I-4]{KM90}) that if an $(\bF, \bP)$-martingale measure $M$ is orthogonal, then there is a random positive finite measure $\mu_M$ on $\cB(\bI \times E)$, which is $\bF$-predictable (i.e. $(\mu_M((0, t] \times A))_{t \in \bI}$ is $\bF$-predictable for all $A \in \cB(E)$), such that
\begin{align*}
	\mu_M((0, t] \times A) = \<M(\cdot, A)\>_t \quad \bP\trm{-a.s.,} \quad \forall (t, A) \in \bI \times \cB(E).
\end{align*} 
The measure $\mu_M$ is then called the \textit{intensity measure} of $M$. Moreover, for $t \in \bI$, $A, B \in \cB(E)$,
\begin{align*}
	\< M(\cdot, A), M(\cdot, B)\>_t = \< M(\cdot, A \cap B)\>_t = \mu_M((0, t] \times (A \cap B)) \quad \bP\trm{-a.s.}
\end{align*}

Let us briefly recall the construction of  stochastic integrals driving by an orthogonal martingale measure $M$ following the It\^o's approach (see \cite{KM90, Wa86}). Define
\begin{align*}
	\bfL^2(\bF, \mu_M) & : = \bb\{H: \cP_\bF \otimes \cB(E)/\cB(\bR) \trm{-measurable} \, \bb| \, \bE\bb[\int_{\bI\times E} H(t, x)^2 \mu_M(\od t, \od x)\bb] < \infty\bb\}.
\end{align*}
For a simple function $H(\omega, t, x) = \sum_{i=1}^n h_{i-1}(\omega) \1_{(t_{i-1}, t_i]}(t) \1_{A_i}(x)$ where $A_i \in \cB(E)$, $0 \le t_0 < t_1 < \cdots < t_n \in \bI$, $h_{i-1}$ is bounded and $\cF_{t_{i-1}}$-measurable, $n \in \bN$, we let
\begin{align*}
	\bdot{H}{M}(t, A) : = \sum_{i=1}^n h_{i-1}[M(t_i \wedge t, A \cap A_i) - M(t_{i-1} \wedge t, A \cap A_i)], \quad (t, A) \in \bI \times \cB(E).
\end{align*}
It is clear that $\bdot{H}{M}$ is an $(\bF, \bP)$-martingale measure and that $\bdot{H}{M}$ satisfies the isometry 
\begin{equation}\label{eq:isometry_martingale_meas}
	\bE\left[|\bdot{H}{M}(t, A)|^2\right]=  \bE\bb[\int_{\bI\times A} H(t, x)^2 \mu_M(\od t, \od x)\bb]. 
\end{equation}
 Since the family of simple functions is dense in $\bfL^2(\bF, \mu_M)$, one can extend $\bdot{H}{M}$ for $H \in \bfL^2(\bF, \mu_M)$ as usual to obtain a martingale measure which is also orthogonal with intensity $\mu_{\bdot{H}{M}}(\od t, \od x) = H(t, x)^2 \mu_M(\od t, \od x)$, see \cite[Theorem I-6]{KM90}. Moreover, \eqref{eq:isometry_martingale_meas} then also holds for $H\in\bfL^2(\bF, \mu_M)$. We often apply the integral notation
$$
\int_{(0, t] \times E} H(s, x) M(\od s, \od x):= \bdot{H}{M}(t,E).
$$

\begin{rema}\label{rema:integral-notation}
	For a martingale measure or an integer-valued random measure $M$ (in the sense of \cite[Definition II.1.3]{JS03}) and  a suitable integrand $H$, we denote the integral process $(\bdot{H}{M}) = ((\bdot{H}{M})_t)_{t \in \bI}$ via
	\begin{align*}
		(\bdot{H}{M})_t : = \int_{(0, t] \times E} H(s, x) M(\od s, \od x)
	\end{align*}
	Notice that the notation $\bdot{H}{M}$ (without brackets) as above stands for a martingale measure.
\end{rema}

\subsection{Proofs for  \cref{sec:discrete_sample_random_measure}}\label{sec:proofs-random-measure}
In this part we let $\bI = [0, T]$.

\subsubsection{Proof of \cref{lem:Bm_Pi}} Let $A \in \cB([0, 1]^d)$. By the definition, $M(0, A) = 0$. For $t \in (0, T]$, we can write
\begin{align*}
	M^{\Pi}_{B^{(l)}}(t, A) = \int_0^t \1_A \bb(\sum_{i=1}^n \1_{(t_{i-1}, t_i]}(s) \xi^\Pi_{t_i} \bb) \od B^{(l)}_s.
\end{align*}
Then, according to \cite[Proposition II-1]{KM90}, $M^\Pi_{B^{(l)}}$ is an orthogonal $(\bF^\Pi, \bP)$-martingale measure on $[0, T] \times \cB([0, 1]^d)$ with intensity $\mu_{B^{(l)}}^\Pi(\od s, \od x) = \delta_{\sum_{i=1}^n \1_{(t_{i-1}, t_i]}(s) \xi^\Pi_{t_i}} (\od x) \od s$. It is clear that $\mu^\Pi_{B^{(l)}} = M^\Pi_D$ as given in \eqref{eq:drift_Pi_measure}. It now suffices to prove the relation \eqref{eq:lem:Bm_Pi} on $(t_{i-1}, t_i]$
for any $\bF^\Pi$-predictable $Y$ satisfying $\bE\b[\int_{t_{i-1}}^{t_i} |Y_s(\xi^\Pi_{t_i})|^2 \od s\b] < \infty$. Assume $Y_s(u) = \sum_{j=1}^k h_{j-1} \1_{(r_{j-1}, r_j]}(s)\1_{A_j}(u)$ for $k \in \bN$, $t_{i-1} \le r_0 < r_1 < \cdots < r_k = t_i$, $A_j \in \cB([0, 1]^d)$, $h_{j-1}$ is bounded and $\cF^{\Pi}_{r_{j-1}}$-measurable. Then, by the definition of $M^{\Pi}_{B^{(l)}}$, one has, a.s.,
\begin{align*}
	&\int_{(t_{i-1}, t_i] \times [0, 1]^d} Y_s(u) M^{\Pi}_{B^{(l)}}(\od s, \od u)  = \sum_{j=1}^k  h_{j-1} \int_{r_{j-1}}^{r_j}  \1_{A_j}(\xi^\Pi_{t_i}) \od B^{(l)}_s = \int_{t_{i-1}}^{t_i} Y_s(\xi^\Pi_{t_i}) \od B^{(l)}_s.
\end{align*}
The conclusion for $Y \in \bfL^2(\bF^\Pi, M^\Pi_D)$ can be derived by a standard approximation argument where one notes that the It\^o isometry coincides for both integrals driven by $M^\Pi_{B^{(l)}}$ and $B^{(l)}$ above. \qed

\subsubsection{Proof of \cref{lem:jumps_Pi}}
By writing $Y = \max\{Y, 0\} - \max\{-Y, 0\}$, we may assume that $Y \ge 0$. By the definition of $M^\Pi_J$, one has, a.s., 
\begin{align*}
	&\int_{(0, T] \times \bR^q_0 \times [0, 1]^d} Y_s(z, u) M^\Pi_J(\od s, \od z, \od u)  = \sum_{i = 1}^n \sum_{s \in (t_{i-1}, t_i]} \1_{\{\Delta L_s \neq 0\}} Y_s(\Delta L_s, \xi^\Pi_{t_i})\\
	& = \sum_{i=1}^n \int_{(0, T] \times \bR^q_0} \1_{(t_{i-1}, t_i]}(s) Y_s(z, \xi^\Pi_{t_i}) N(\od s, \od z), \quad t \in [0, T],
\end{align*}
which then verifies \eqref{eq:lem:jumps_Pi-1}. Moreover, as $\nu_s(\od z) \od s$ is the $(\bF^\Pi, \bP)$-predictable compensator of $N(\od s, \od z)$ (see \cite[Proposition II.1.21]{JS03}), we get
\begin{align*}
	\bE[(\bdot{Y}{M^\Pi_J})_T]	& = \bE\bb[\int_{(0, T] \times \bR^q_0 \times [0, 1]^d} Y_s(z, u) M^\Pi_J(\od s, \od z, \od u)\bb]  \\
	& = \bE\bb[\sum_{i=1}^n \int_{(0, T] \times \bR^q_0} \1_{(t_{i-1}, t_i]}(s) Y_s(z, \xi^\Pi_{t_i}) \nu_s(\od z) \od s\bb] \\
	& = \bE\bb[ \sum_{i=1}^n \int_{(0, T] \times \bR^q_0 \times [0, 1]^d}  Y_s(z, u) \1_{(t_{i-1}, t_i]}(s) \delta_{\xi^\Pi_{t_i}}(\od u) \nu_s(\od z) \od s\bb]\\
	&  = \bE\bb[\int_{(0, T] \times \bR^q_0 \times [0, 1]^d}  Y_s(z, u) \mu^{\Pi}_J (\od s, \od z, \od u)\bb] \\
	& = \bE[(\bdot{Y}{\mu^\Pi_J})_T].
\end{align*}
We note that $(\bdot{Y}{\mu^\Pi_J})$ is $\bF^\Pi$-predictable as the pointwise limit of the continuous and $\bF^\Pi$-adapted processes $(\bdot{Y^n}{\mu^\Pi_J})$ as $n \to \infty$ where $Y^n: = (Y\wedge n)\1_{\{|z| > 1/n\}}$. Hence, $\mu^\Pi_J$ is an $\bF^\Pi$-predictable random measure in the sense of \cite[Definition II.1.6(a)]{JS03}. By \cite[Theorem II.1.8(i)]{JS03}, we conclude that $\mu^\Pi_J$ is the $(\bF^\Pi, \bP)$-predictable compensator of $M^\Pi_J$.

The relation \eqref{eq:lem:jumps_Pi-2} can be achieved in the usual way by first proving for $(-n \vee Y \wedge n)\1_{\{|z| > 1/n\}}$ in place of $Y$, and then taking the limit in $\bfL^2(\bP)$ when $n \to \infty$ with the aid of It\^o's isometry. \qed

\subsubsection{Proof of \cref{lem:Bm_from_white_noise}} The assumption $\int_{[0, 1]^d} \eta^{(k)}_s(u) \eta^{(k')}_s(u) \od u = \1_{\{k = k'\}}$ for $\bP \otimes \Leb_{[0, T]}$-a.e. $(\omega, s) \in \Omega \times [0, T]$ particularly implies that $\bE \b[\int_0^T \int_{[0, 1]^d} |\eta^{(k)}_s(u)|^2 \od u \od s \b] = T$. Hence, for any $(k, l)$, $(\bdot{\eta^{(k)}}{M_{B^{(l)}}})$ is a square integrable $(\bF, \bP)$-martingale null at $0$. Since $M_{B^{(l)}}$ is a continuous martingale measure (see \cite[Section II(3)]{KM90}), the process $(\bdot{\eta^{(k)}}{M_{B^{(l)}}})$ is also continuous as indicated in \cite[Propisition I-6(1)]{KM90}. As $M_{B^{(l)}}$ and $M_{B^{(l')}}$ are independent for $l \neq l'$ by assumption, it is straightforward to prove that the product $(\bdot{\eta^{(k)}}{M_{B^{(l)}}}) (\bdot{\eta^{(k')}}{M_{B^{(l')}}})$ is also a continuous $(\bF, \bP)$-martingale, which thus implies that $\langle (\bdot{\eta^{(k)}}{M_{B^{(l)}}}), (\bdot{\eta^{(k')}}{M_{B^{(l')}}}) \rangle = 0$. We compute the quadratic covariation using \cite[Proposition I-6(2)]{KM90}, a.s.,
\begin{align*}
	\big\langle (\bdot{\eta^{(k)}}{M_{B^{(l)}}}), (\bdot{\eta^{(k')}}{M_{B^{(l')}}}) \big\rangle_t = \1_{\{l = l'\}} \int_0^{t}\nq \int_{[0, 1]^d} \eta^{(k)}_s(u) \eta^{(k')}_s(u) \od u \od s =  \1_{\{(k, l) = (k', l')\}} t.
\end{align*}
Therefore, the desired conclusion follows from the  L\'evy characterization for Brownian motion. \qed

\section{Miscellaneous}\label{app:proofs}
\subsection{Proof of \cref{prop:value}} \label{app:proof-value}
	(1) Recall from \cref{sec:exploratory}, that the law of $X^{\bfh}$ solves the martingale problem for the operator $\cL_{h}$. Hence,
	$$
	J(t,X^{\bfh}_t)-\int_0^t  \bb(\frac{\partial J}{\partial t}(s,X^{\bfh}_s)+(\cL_{h} J(s,\cdot))(s,X^{\bfh}_s)\bb)\od s
	$$
	is a local martingale. Inserting the partial differential equation, we observe that
	\begin{equation*}
		J(t,X^{\bfh}_t)+\lambda \int_0^t \nq \int_{\bR} \dot h(s,X^{\bfh}_s,y)\log\dot h(s,X^{\bfh}_s,y)\od y \od s
	\end{equation*}
	is a local martingale, and hence a martingale, by the boundedness assumptions on $J$ and on the entropy. Thus, a.s.,
	\begin{align*}
		J(t,X^{\bfh}_t) & =\bE\bb[ J(T,X^{\bfh}_T)+\lambda \int_0^T \nq \int_{\bR} \dot h(s,X^{\bfh}_s,y)\log\dot h(s,X^{\bfh}_s,y)\od y \od s \,\bb| \,\mathcal{F}_t \bb]\\
		& \quad -\lambda \int_0^t \nq  \int_{\bR} \dot h(s,X^{\bfh}_s,y)\log\dot h(s,X^{\bfh}_s,y)\od y \od s \\
		& = \bE\bb[g(X^{\bfh}_T) +\lambda \int_t^T \nq \int_{\bR} \dot h(s,X^{\bfh}_s,y) \log\dot h(s,X^{\bfh}_s,y)\od y \od s\,\bb| \,\cF_t \bb]\\
		& =\cJ_t^{\bfh},
	\end{align*} 
	i.e., $J$ is a value function of ${\bfh}$.
	
	\smallskip
	
	 (2)
	If $\tilde J$ is a value function of ${\bfh}$, then $(\tilde J(t,X^{\bfh}_t))_{t\geq 0}$ is a modification of $\mathcal{J}$. Hence,
	\begin{equation}\label{eq:value_martingale}
		\tilde J(t,X^{\bfh}_t)+\lambda \int_0^t \nq \int_{\bR} \dot h(s,X^{\bfh}_s,y) \log\dot h(s,X^{\bfh}_s,y)\od y \od s
	\end{equation}
	inherits the martingale property of 
	$$
	\mathcal{J}_t^{\bfh}+\lambda \int_0^t \nq \int_{\bR}\dot h(s,X^{\bfh}_s,y) \log\dot h(s,X^{\bfh}_s,y)\od y \od s.
	$$
 Conversely, if the process in   \eqref{eq:value_martingale}  is a martingale, then the last part of the proof of (1) can be repeated with $\tilde J$ in place of $J$ to conclude that $\tilde J$ is a value function of ${\bfh}$.
\qed

\subsection{Proof of \cref{lemm:characteristic-limit}} \label{sec:proof-semimartingale-characteristics}
 Recall the representation of $\cX$ in \cref{thm:limit}. For $l = 1, \ldots, p$, \cite[Section II(2)]{KM90} asserts that $\int_0^{\cdot} \int_{[0, 1]^d} f^{(k)}_l(s, u) M_{B^{(l)}}(\od s, \od u)$ is a continuous square integrable martingale with quadratic variation $\int_0^{\cdot} \int_{[0, 1]^d} |f_l^{(k)} (s, u)|^2 \od u \od s$. The boundedness of $f_{p+1}, f_{p+2}$ and \eqref{assumption:Levy-measure} imply
	\begin{align*}
		\int_0^T \nq \int_{\bR^q_0 \times [0, 1]^d} [|f_{p+1}(s, z, u)|^2 |z|^2 \1_{\{0 < |z| \le R\}} + |f_{p+2}(s, z, u)| \1_{\{|z| >R\}}] \mu_J(\od s, \od z, \od u) < \infty,
	\end{align*}
	which shows that the process driven by $\tilde M_J$ is a square integrable martingale, and that against $M_J$ is an a.s. finite variation process.  Hence, $\cX$ is an $\bR^m$-valued semimartingale.
	
	According to \cite[Proposition I-6]{KM90}, the quadratic covariation matrix of the continuous martingale part of $\cX$ is
	\begin{align*}
		& \left\langle \sum_{l=1}^p \int_0^{\cdot} \nq \int_{[0, 1]^d} f^{(k)}_l(s, u) M_{B^{(l)}}(\od s, \od u), \sum_{l' =1}^ p\int_0^{\cdot} \nq \int_{[0, 1]^d} f^{(k')}_{l'}(s, u) M_{B^{(l')}}(\od s, \od u)  \right\rangle \\
		& = \sum_{l=1}^ p \int_0^{\cdot} \nq \int_{[0, 1]^d} (f^{(k)}_l f^{(k')}_l)(s, u) \od u \od s = C^{\cX, (k, k')}.
	\end{align*}
	For the jump part, it follows from
	\cite[Ch.3, Theorem 1]{LS89} that
	\begin{align*}
		\Delta \cX_r = \int_{\{r\} \times \bR^q_0 \times [0,1]^d} [f_{p+1}(s, z, u) |z| \1_{\{0 < |z| \le R\}} + f_{p+2}(s, z, u) \1_{\{|z| > R\}} ]  M_J(\od s, \od z, \od u), \;r \in [0, T]\; \bP\trm{-a.s.}
	\end{align*}
	Let $A \in \cB(\bR^m_0)$ with $A \cap B_m(\kappa) = \emptyset$ for some $\kappa >0$ where $B_m(\kappa) = \{y \in \bR^m : |y| < \kappa\}$. Since $f_{p+1}$ is bounded, there exists $\ep >0$ sufficiently small such that 
	\begin{align*}
		&\left\{(r, z, u) : f_{p+1}(r, z, u) |z|\1_{\{0< |z| \le R \}} + f_{p+2}(r, z, u) \1_{\{|z| > R \}} \in A \right\} \\
		& = \left\{(r, z, u) : f_{p+1}(r, z, u) |z|\1_{\{\ep < |z| \le R \}} + f_{p+2}(r, z, u) \1_{\{|z| > R \}} \in A \right\}.
	\end{align*}
	We define the process $(L^Z, L^U)$ depending on $\ep$ via
	\begin{align*}
		(L^Z_t, L^U_t) : = \int_{(0, t] \times \{|z| > \ep\} \times [0, 1]^d} (z, u) M_J(\od s, \od z, \od u), \quad t \in [0, T].
	\end{align*}
	Let $N_{\cX}$ be the random jump measure of $\cX$. Then
	\begin{align*}
		N_{\cX}((s, t] \times A)  & = \sum_{s < r \le t} \1_{\{\Delta \cX_r \in A\}}\\
		&  =  \sum_{s < r \le t} \1_{\left\{ f_{p+1}(r, \Delta L_r^{Z}, \Delta L_r^{U}) |\Delta L^{Z}_r|\1_{\{\ep < |\Delta L^{Z}_r| \le R \}} + f_{p+2}(r, \Delta L_r^{Z}, \Delta L_r^{U}) \1_{\{|\Delta L^{Z}_r| > R \}} \in A \right\}} \\
		&  = \int_s^t \nq \int_{\bR^q_0 \times [0, 1]^d} \1_{A}\b(f_{p+1}(r, z, u)|z| \1_{\{\ep < |z| \le R\}} + f_{p+2}(r, z, u)\1_{\{|z| >R\}} \b) M_J (\od r, \od z, \od u)\\
		& = \int_s^t \nq \int_{\{0 < |z| \le R\} \times [0, 1]^d} \1_{A}(f_{p+1}(r, z, u)|z|) M_J(\od r, \od z, \od u) \\
		& \quad + \int_s^t \nq \int_{\{|z| > R\} \times [0,1]^d} \1_{A}(f_{p+2}(r, z, u)) M_J(\od r, \od z, \od u).
	\end{align*}
	Since $\mu_J(\od r, \od z, \od u) = \nu_r(\od z) \od u \od r$ is the predictable compensator of $M_J(\od r, \od z, \od u)$, it implies that
	\begin{align*}
		\nu^{\cX}((s, t] \times A)
		& = \int_s^t \nq \int_{\{0 < |z| \le R\} \times [0, 1]^d} \1_{A}(f_{p+1}(r, z, u)|z|) \nu_r (\od z) \od u \od r\\
		&  \quad + \int_s^t \nq \int_{\{|z| > R\} \times [0, 1]^d} \1_{A}(f_{p+2}(r, z, u)) \nu_r (\od z) \od u \od r.
	\end{align*}
	This result can be extended to $A \in \cB(\bR^m_0)$ by using the approximation sequence $(A \cap B_m(\frac{1}{n}))_{n\in \bN}$. 
	For the predictable finite variation part $\frb^\cX$, one has, a.s., 
	\begin{align*}
		\cY_t & := \cX_t -  \sum_{l=1}^p \int_0^t \nq \int_{[0, 1]^d} f_l (s, u) M_{B^{(l)}}(\od s, \od u) - \int_0^t \nq \int_{\bR^m_0} (y - \frh(y)) N_{\cX}(\od s, \od y) \\
		& =  \int_0^t \nq \int_{[0, 1]^d} f_0(s, u) \od u \od s + \int_0^t \nq \int_{\{0 < |z| \le R\}  \times [0, 1]^d} f_{p+1}(s, z, u) |z| \tilde M_J (\od s, \od z, \od u) \\
		& \quad  + \int_0^t \nq \int_{\{|z| > R\} \times [0, 1]^d} f_{p+2}(s, z, u) M_J (\od s, \od z, \od u) \\
		& \quad - \int_0^t \nq \int_{\{0 < |z| \le R\} \times [0, 1]^d} [f_{p+1}(s, z, u)|z| - \frh(f_{p+1}(s, z, u)|z|)] M_J (\od s, \od z, \od u) \\
		& \quad - \int_0^t \nq \int_{\{|z| > R\} \times [0, 1]^d} [f_{p+2}(s, z, u) - \frh(f_{p+2}(s, z, u))] M_J (\od s, \od z, \od u)\\
		& = \int_0^t \nq \int_{[0, 1]^d} f_0(s, u) \od u \od s + \int_0^t \nq \int_{\{|z| >R\} \times [0, 1]^d} \frh(f_{p+2}(s, z, u)) \nu_s(\od z) \od u \od s\\
		& \quad - \int_0^t \nq \int_{\{0 < |z| \le R\} \times [0, 1]^d} [f_{p+1}(s, z, u)|z| - \frh(f_{p+1}(s, z, u)|z|)]  \nu_s(\od z) \od u \od s\\
		& \quad + \int_0^t \nq \int_{\{0 < |z| \le R\}  \times [0, 1]^d} f_{p+1}(s, z, u) |z| \tilde M_J (\od s, \od z, \od u) \\
		& \quad + \int_0^t \nq \int_{\{|z| >R\} \times [0, 1]^d} \frh(f_{p+2}(s, z, u)) \tilde M_J (\od s, \od z, \od u)\\
		& \quad  - \int_0^t \nq \int_{\{0 < |z| \le R\} \times [0, 1]^d} [f_{p+1}(s, z, u)|z| - \frh(f_{p+1}(s, z, u)|z|)] \tilde M_J (\od s, \od z, \od u),
	\end{align*}
	where in the last equality we use the fact that $\int F \tilde M_J = \int F M_J - \int F \mu_J$ if  $F$ is predictable and $\mu_J$-integrable, see \cite[Proposition II.1.28]{JS03}. By identifying the predictable finite variation component of $\cY$, we obtain the desired expression of $\frb^{\cX}$.
\qed

\subsection{On the Poisson random measure $M_J$}\label{sec:random_measure_construction}
Assume the L\'evy process as defined in \eqref{def:induced-Levy-process}. Let $\{T^n_j\}_{n, j \ge 0}$ be the of jump times of $L$ given by
		\begin{align*}
			&T^0_0: = 0, \quad T^0_j : = \inf\{t > T^0_{j-1} : |\Delta L_t| >1\}, \quad j \ge 1,\\
			&T^n_0: = 0, \quad T^n_j : = \inf\{t > T^n_{j-1} : 1/(n+1) < |\Delta L_t| \le 1/n\}, \quad j \ge 1, \quad n \ge 1.
		\end{align*}
		Let $\{\xi^n_j\}_{n, j \ge 0}$ be i.i.d. with uniform distribution on $[0, 1]^d$. Assume that $\{\xi^n_j\}_{n, j \ge 0}$ is independent of $L$. We define the Poisson random measure $M_J$ on $[0, T] \times \bR^q_0 \times [0, 1]^d$ by
		\begin{align*}
			M_J(\omega, \od t, \od z, \od u) = \sum_{n = 0}^ \infty \sum_{j = 1}^\infty  \delta_{(T^n_j(\omega), \Delta L_{T^n_j(\omega)}(\omega), \xi^n_j(\omega))}(\od t, \od z, \od u).
		\end{align*}
		We note that, in general, there is no semimartingale which possesses $M_J$ as the associated random jump measure because $\int_0^T \! \int_{0 < |z|^2 + |u|^2 \le 1} (|z|^2 + |u|^2) \mu_J(\od t, \od z, \od u)$ might be infinite, except the case $\int_0^T \! \int_{\bR^q_0} \nu_t(\od z) \od t < \infty$ (i.e. $L$ is of finite activities).

\subsection{On the independence of $(M_{B^{(1)}}, \ldots, M_{B^{(p)}})$ and $M_J$}\label{rema:independence}
Assume that $(M_{B^{(1)}}, \ldots, M_{B^{(p)}})$ and $M_J$ define on the same probability space, then
	$$\bb\{\int_{(0, T] \times [0, 1]^d} g_l(s, u) M_{B^{(l)}}(\od s, \od u) \,\bb|\, g_l \colon [0, T] \times [0, 1]^d \to \bR \trm{ measurable and bounded}, l = 1, \ldots, p\bb\}$$ is independent of $$\bb\{\int_{(0, T] \times \bR^q_0 \times [0, 1]^d} h(s, z, u) M_J(\od s, \od z, \od u) \,\bb|\, h \colon [0, T] \times \bR^q_0 \times [0, 1]^d \to [0, \infty) \trm{ measurable}\bb\}.$$
	
	 Indeed, it is sufficient to show that $G = (\sum_{l=1}^p \int_0^t\int_{[0, 1]^d} g_l(s, u) M_{B^{(l)}}(\od s, \od u))_{t \in [0, T]}$ is independent of $H = (\int_{(0, t] \times \{|z| > \kappa\} \times [0, 1]^d} h(s, z, u) M_J(\od s, \od z, \od u))_{t \in [0, T]}$ for all (non-random) measurable and bounded $g_l$, $h \ge 0$ and $\kappa >0$. It is clear that $H$ is of finite variation and $G$ is a continuous martingale (see \cite[Section II(3)]{KM90}), and both are processes with independent increments.  Observe that $[G, H]_t = \sum_{0 \le s \le t} \Delta G_s \Delta H_s = 0$ for $t \in [0, T]$ a.s. It then follows from \cite[Theorem 11.43]{HWY92} that $G$ and $H$ are independent.

\section{Weak convergence in the Skorokhod topology and Jacod--Shiryaev's limit theorem for triangular arrays}\label{app:Skorokhod}

\subsection{Skorokhod spaces and weak convergence}\label{sec:Skorokhod-space}
 Fix $T \in (0, \infty)$ and let $\bD_T(\bR^m)$ be the family of all c\`adl\`ag functions $f \colon [0, T] \to \bR^m$ and $\Lambda_T$ consists of all strictly increasing and continuous $\lambda \colon [0, T] \to [0, T]$ with $\lambda(0) = 0$, $\lambda(T) = T$. We equip $\bD_T(\bR^m)$ with the Skorokhod metric
\begin{align*}
	d^{m}_T(x, y) := \inf_{\lambda \in \Lambda_T} \max\bb\{\sup_{0 \le s < t \le T} \bb|\log\frac{\lambda(t) - \lambda(s)}{t-s}\bb|, \;\sup_{0 \le t \le T} |x(t) - y(\lambda(t))|\bb\}.
\end{align*}
It is well-known that $(\bD_T(\bR^m), d^{m}_T)$ is a complete and separable metric space (see  \cite[Section 14]{Bi99}), however, it is not a topological vector space. It is also convenient to work with the metric $\tilde d^m_T$, which defines the same topology as $d^m_T$ does, given by
\begin{align*}
	\tilde d^{m}_T(x, y) := \inf_{\lambda \in \Lambda_T} \max\bb\{\sup_{0 \le t \le T} |\lambda(t) -t|, \;\sup_{0 \le t \le T} |x(t) - y(\lambda(t))|\bb\}.
\end{align*}
However, $(\bD_T(\bR^m), \tilde d^{m}_T)$ is not complete.

An $\bR^m$-valued c\`adl\`ag process $X = (X_t)_{t \in [0, T]}$ can be regarded as an $\cF/\cB(\bD_T(\bR^m))$-measurable function $X \colon \Omega \to \bD_T(\bR^m)$ where $\cB(\bD_T(\bR^m))$ is the Borel $\sigma$-algebra induced by the Skorokhod metric $d^m_T$. A sequence of $\bR^m$-valued c\`adl\`ag processes $(X^n)_{n \in \bN}$, where $X^n$ is defined on $(\Omega^n, \cF^n, \bP^n)$, is said to be \textit{weakly convergent} to a c\`adl\`ag process $X$ defined on $(\Omega, \cF, \bP)$
if 
\begin{align*}
	\bE^n[f(X^n)] \xrightarrow{n \to \infty} \bE[f(X)], \quad \forall f \in C_b(\bD_T(\bR^m)),
\end{align*}
where $\bE^n$ and $\bE$ are the expectation under $\bP^n$ and $\bP$, respectively. We then write 
$X^n \xrightarrow{\scrD_T} X$.

\subsection{A limit theorem of Jacod--Shiryaev for triangular arrays}\label{sec:JS03-limit-theorem}
For the reader's convenience, we recall (and adapt to our setting) a limit theorem establishing the weak convergence of triangular arrays which we use to prove the main result in this article.

Let $(\Omega, \cF, \bP)$ be a complete probability space and suppose that $\{U^n_i, \cG^n_i : i \ge 0\}$, $n \in \bN$,  are adapted sequences of $\bR^d$-valued random variables. For each $n \in \bN$, we consider a change of time $\sigma_n \colon \Omega \times [0, \infty) \to [0, \infty)$ with respect to $(\cG^n_i)_{i \ge 0}$, i.e.,
\begin{enumerate}[(a)]
	\item $\sigma_n(\cdot, 0) = 0$;
	
	\item For any $\omega$, $\sigma_n(\omega, \cdot)$ is increasing, right-continuous, with jumps equal to 1;
	
	\item For any $t\ge 0$, $\sigma_n(\cdot, t)$ is a $(\cG^n_i)_{i \ge 0}$-stopping time.
\end{enumerate}

\begin{theo}[\cite{JS03}, Theorem VIII.2.29] \label{thm:weak-convergence-JS03}
	Assume a sequence of $d$-dimensional semimartingales $(X^n)_{n \in \bN}$ where $X^n_t = \sum_{i=1}^{\sigma_n(t)} U^n_i$, $t \ge 0$. Let $X$ be a $d$-dimensional process with independent increments and without fixed time of discontinuity, having characteristics $(\frb, C, \nu)$ in relation to a truncation function $\frh$. Set
	$\wt C_t^{(k, l)} : = C_t^{(k, l)} + \int_0^t \int_{\bR^d} (\frh^{(k)}\frh^{(l)})(y) \nu(\od s, \od y)$ as in \cite[II.5.8]{JS03}. If there exists some dense subset $D$ of $[0, \infty)$ such that, as $n \to \infty$,
	\begin{align*}
		&\sup_{0 \le s \le t} \bb| \sum_{i=1}^{\sigma_n(s)} \bE[\frh(U^n_i)|\cG^n_{i-1}] - \frb_s\bb| \xrightarrow{\bP} 0 \quad \forall t \ge 0,\\
		& \sum_{i=1}^{\sigma_n(t)} \B(\bE[(\frh^{(k)} \frh^{(l)})(U^n_i)|\cG^n_{i-1}] - \bE[\frh^{(k)}(U^n_i)|\cG^n_{i-1}]\, \bE[ \frh^{(l)}(U^n_i)|\cG^n_{i-1}]\B) \xrightarrow{\bP} \wt C^{(k, l)}_t \quad \forall t \in D,\\
		& \sum_{i=1}^{\sigma_n(t)} \bE[g(U^n_i)|\cG^n_{i-1}] \xrightarrow{\bP} \int_0^t \nq \int_{\bR^d} g(y) \nu(\od s, \od y) \quad \forall t \in D, g \in C_1(\bR^d),
	\end{align*}
	then $X^n$ converges weakly to $X$ in the Skorokhod topology on the space $\bD_\infty(\bR^d)$ of c\`adl\`ag functions $F \colon [0, \infty) \to \bR^d$. Here, $C_1(\bR^d)\subset C_2(\bR^d)$ is a particular class of test functions vanishing around zero and is introduced in \cite[VII.2.7]{JS03}.
\end{theo}

\bibliographystyle{amsplain}

\end{document}